\newif\ifjmlr
\jmlrfalse

\ifjmlr
\documentclass[twoside,11pt]{article}
\usepackage{jmlr2e}
\jmlrheading{?}{20??}{?--?}{7/11}{?/?}{Matus Telgarsky}
\ShortHeadings{A Primal-Dual Convergence Analysis of Boosting}{Telgarsky}
\firstpageno{1}
\newcommand{\jmlrcref}[1]{\Cref{#1}}
\renewenvironment{proof}[1][Proof]{\par\noindent{\bf {#1}\ }}{\hfill\BlackBox\\[2mm]}

\newcommand{\qedhere}{}
\else
\documentclass[reqno]{article}
\usepackage[margin=1.25in]{geometry}
\usepackage[round,comma]{natbib}
\usepackage{amsthm}
\newcommand{\jmlrcref}[1]{\Cref{#1}}
\fi

\newcommand{\jmlreqref}[1]{\eqref{#1}}

\usepackage{amsmath}
\usepackage{amssymb,graphicx,dsfont}
\ifjmlr
\usepackage{thmtools}
\else
\usepackage{thmtools}
\fi
\usepackage{subfig} 
\usepackage[usenames,dvipsnames]{color}
\usepackage{ifpdf}
\ifpdf
\usepackage{hyperref}
\else
\fi
\usepackage{cleveref}
\usepackage{mathrsfs}

\usepackage{enumerate} 

\ifjmlr
\fi
\usepackage[normalem]{ulem}
\usepackage{environ}
\ifjmlr
\else
\fi

\def\1{\mathds 1}
\def\R{\mathbb R}
\def\bfe{\mathbf e}
\def\bfo{\mathbf 1}
\def\bfz{\mathbf 0}
\def\bG{\mathbb G}
\def\cC{\mathcal C}
\def\cH{\mathcal H}
\def\cI{\mathcal I}

\def\cO{\mathcal O}
\def\cS{\mathcal S}

\def\cX{\mathcal X}
\def\cY{\mathcal Y}

\def\sfD{\mathsf D}
\def\sfP{\mathsf P}

\def\scrH{\mathscr H}

\def\dist{\sfD}
\def\proj{\sfP}
\def\nf{\nabla f}
\def\ri{\textup{ri}}
\def\aff{\textup{aff}}
\def\cont{\textup{cont}}

\def\cl{\textup{cl}}
\def\dom{\textup{dom}}
\def\Im{\textup{Im}}
\def\Ker{\textup{Ker}}
\def\Boost{\ensuremath{\textsc{Boost}}}
\def\Wolfe{\ensuremath{\textsc{Wolfe}}}

\def\interior{\textup{int}}
\newcommand{\ip}[2]{\left\langle #1, #2 \right \rangle}
\newcommand{\argmin}{\operatornamewithlimits{argmin}}
\newcommand{\Argmin}{\operatornamewithlimits{Argmin}}
\newcommand{\argmax}{\operatornamewithlimits{argmax}}

\ifjmlr
\crefname{jtheorem}{Theorem}{Theorems}
\Crefname{jtheorem}{Theorem}{Theorems}
\declaretheorem[name=Theorem,refname=Theorem,Refname=Theorem]{jtheorem}
\crefname{jexample}{Example}{Examples}
\Crefname{jexample}{Example}{Examples}
\declaretheorem[name=Example,refname=Example,Refname=Example]{jexample}
\crefname{jlemma}{Lemma}{Lemmas}
\Crefname{jlemma}{Lemma}{Lemmas}
\declaretheorem[numberlike=jtheorem,name=Lemma,refname=Lemma,Refname=Lemma]{jlemma}
\crefname{jproposition}{Proposition}{Propositions}
\Crefname{jproposition}{Proposition}{Propositions}
\declaretheorem[numberlike=jtheorem,name=Proposition,refname=Proposition,Refname=Proposition]{jproposition}
\crefname{jcorollary}{Corollary}{Corollaries}
\Crefname{jcorollary}{Corollary}{Corollaries}
\declaretheorem[numberlike=jtheorem,name=Corollary,refname=Corollary,Refname=Corollary]{jcorollary}
\crefname{jdefinition}{Definition}{Definitions}
\Crefname{jdefinition}{Definition}{Definitions}
\declaretheorem[numberlike=jtheorem,name=Definition,refname=Definition,Refname=Definition]{jdefinition}
\crefname{jremark}{Remark}{Remarks}
\Crefname{jremark}{Remark}{Remarks}
\declaretheorem[numberlike=jtheorem,name=Remark,refname=Remark,Refname=Remark]{jremark}
\else
\numberwithin{equation}{section}

\crefname{jtheorem}{theorem}{theorems}
\Crefname{jtheorem}{Theorem}{Theorems}
\crefname{jexample}{example}{examples}
\Crefname{jexample}{Example}{Examples}
\crefname{jlemma}{lemma}{lemmas}
\Crefname{jlemma}{Lemma}{Lemmas}
\crefname{jproposition}{proposition}{propositions}
\Crefname{jproposition}{Proposition}{Propositions}
\crefname{jcorollary}{corollary}{corollaries}
\Crefname{jcorollary}{Corollary}{Corollaries}
\crefname{jdefinition}{definition}{definitions}
\Crefname{jdefinition}{Definition}{Definitions}
\crefname{jremark}{remark}{remarks}
\Crefname{jremark}{Remark}{Remarks}

\declaretheorem[numberlike=equation,name=Theorem,refname=Theorem,Refname=Theorem]{jtheorem}
\declaretheorem[numberlike=jtheorem,name=Lemma,refname={lemma,lemmas},Refname={Lemma,Lemmas}]{jlemma}
\declaretheorem[numberlike=jtheorem,name=Proposition,refname=proposition,Refname=Proposition]{jproposition}

\declaretheoremstyle[%
qed={\ensuremath\Diamond}]{remstyle}
\declaretheorem[numberlike=jtheorem,style=remstyle,name=Definition,refname=definition,Refname=Definition]{jdefinition}
\declaretheorem[numberlike=jtheorem,style=remstyle,name=Remark,refname=remark,Refname=Remark]{jremark}
\declaretheorem[numberlike=jtheorem,style=remstyle,name=Example,refname=example,Refname=Example]{jexample}
\fi

\title{A Primal-Dual Convergence Analysis of Boosting}

\ifjmlr
\author{\name Matus Telgarsky \email mtelgars@cs.ucsd.edu\\
    \addr Department of Computer Science and Engineering\\
    University of California, San Diego\\
    San Diego, CA 92093-0404, USA}

\editor{(unassigned)}
\else
\author{Matus Telgarsky\thanks{Department of Computer Science and Engineering, University of California, San Diego.
Email:
\texttt{<mtelgars@cs.ucsd.edu>}.}}
\date{}
\fi

\begin{document}

\maketitle
\begin{abstract}
Boosting 
combines weak learners into a predictor with low empirical risk.
Its dual constructs a high entropy distribution upon which
weak learners and training labels are uncorrelated.
This manuscript studies this primal-dual relationship under a broad
family of losses, including
the exponential loss of AdaBoost and the logistic loss, revealing:
\ifjmlr
\begin{minipage}{0.85\textwidth}
\vspace{1mm} 
\else
\fi
\begin{itemize}
\item
Weak learnability aids the whole loss family:
for any $\epsilon > 0$, $\cO(\ln(1/\epsilon))$
iterations suffice to produce a predictor with empirical risk $\epsilon$-close to
the infimum;
\item
The circumstances granting the existence of an empirical risk minimizer may be
characterized in terms of the primal and dual problems, yielding a new proof
of the known rate $\cO(\ln(1/\epsilon))$;
\item
Arbitrary instances may be decomposed into the above two, granting
rate $\cO(1/\epsilon)$, with a matching lower bound provided for the
logistic loss.
\end{itemize}
\ifjmlr
\end{minipage}
\fi
\end{abstract}
\ifjmlr
\vspace{1mm} 
\begin{keywords}
Boosting, convex analysis, weak learnability,
coordinate descent, maximum entropy.
\end{keywords}
\fi

\section{Introduction}
Boosting is the task of converting inaccurate \emph{weak learners} into a single accurate
predictor.
The existence of any such method was unknown until the breakthrough result of
\citet{schapire_wl}:  under a \emph{weak learning assumption}, it is possible
to combine many carefully chosen weak learners into a majority of majorities with 
arbitrarily low training error.  Soon after, \citet{yoav_boost_by_majority}
noted that a single majority is enough,
and that $\Theta(\ln(1/\epsilon))$ iterations are both necessary and sufficient 
to attain accuracy $\epsilon$.  Finally, their combined effort produced AdaBoost,
which exhibits this optimal convergence rate (under the weak learning assumption),
and has an astonishingly simple implementation \citep{freund_schapire_adaboost}.

It was eventually revealed that AdaBoost was minimizing a risk functional, specifically
the exponential loss \citep{breiman}.  Aiming to alleviate perceived
deficiencies in the algorithm, other loss functions were proposed, foremost amongst
these being the logistic loss
\citep{friedman_hastie_tibshirani_statboost}.  Given the wide practical success
of boosting with the logistic loss, it is perhaps surprising that 
no convergence rate better than $\cO(\exp(1/\epsilon^2))$ was known, even under the weak
learning assumption \citep{bickel_ritov_zakai}.
The reason for this deficiency is simple: unlike SVM, least squares, and basically
any other optimization problem considered in machine learning, there might not exist a
choice which attains the minimal risk!
This reliance is carried over from convex optimization, where the
assumption of attainability is generally made, either directly, or through stronger
conditions like compact level sets or strong 
convexity~\citep{luo_tseng_coordinate_descent}.
But this limitation seems artificial: a function like $\exp(-x)$ has no minimizer but
decays rapidly.

Convergence rate analysis provides a valuable mechanism to
compare and improve of minimization algorithms. 
But there is a deeper significance with boosting:
a convergence rate of $\cO(\ln(1/\epsilon))$ means that, with a combination of just
$\cO(\ln(1/\epsilon))$ predictors, one can construct an $\epsilon$-optimal classifier, 
which is crucial to both the computational efficiency and statistical stability
of this predictor.


The main contribution of this manuscript is to provide a tight convergence theory for a large
family of losses, including the exponential and logistic losses, which has
heretofore resisted analysis.
In particular,
it is shown that 
the (disjoint) scenarios of weak learnability (\jmlrcref{sec:rate:wl}) and attainability
(\jmlrcref{sec:rate:attainable}) both exhibit the rate $\cO(\ln(1/\epsilon))$.  These
two scenarios are in a strong sense extremal, and general instances
are shown to decompose 
into them; but their conflicting behavior 
yields a degraded rate $\cO(1/\epsilon)$ (\jmlrcref{sec:rate:general}). A matching lower bound for
the logistic loss demonstrates this is no artifact.



\subsection{Outline}
Beyond providing these rates, this manuscript will study the rich ecology within the
primal-dual interplay of boosting.

Starting with necessary background, \jmlrcref{sec:setup} provides
the standard view of boosting as coordinate descent of an empirical risk.
This primal formulation of boosting obscures a key internal mechanism: boosting iteratively
constructs distributions where the previously selected weak learner fails.  This view
is recovered in the dual problem; specifically, \jmlrcref{sec:dual} reveals that the dual
feasible set is the collection of distributions where all weak learners have no correlation
to the target, and the dual objective is a max entropy rule.

The dual optimum is always attainable; since a standard mechanism in convergence analysis to
control the distance to the optimum, why not overcome the unattainability of the primal optimum by
working in the dual?  It turns out that the classical
weak learning rate was a mechanism to control distances in the dual all along; 
by developing a suitable generalization (\jmlrcref{sec:wl}), it is possible to convert the
improvement due to a single step of coordinate descent into a relevant distance in the
dual (\jmlrcref{sec:rate}).  Crucially, this holds for general instances, without any
assumptions.

The final puzzle piece is to relate these dual distances to the optimality gap.
\Cref{sec:hard_core} lays the foundation, taking a close look at the structure
of the optimization problem.  The classical scenarios of attainability and weak learnability
are identifiable directly from the weak learning class and training sample; moreover, they
can be entirely characterized by properties of the primal and dual problems.

\Cref{sec:hard_core} will also reveal another structure: there is a subset of the
training set, the \emph{hard core}, which is the maximal support of any distribution upon
which every weak learner and the training labels are uncorrelated.  This set is central---for 
instance, the dual optimum (regardless of the loss function) places positive weight on exactly the hard core.
Weak learnability corresponds to the hard core being empty, and attainability corresponds to it
being the whole training set.  For those instances where the hard core
is a nonempty proper subset of the training set,
the behavior on and off the hard core mimics attainability and weak
learnability, and \jmlrcref{sec:rate:general} will leverage this to produce rates using facts
derived for the two constituent scenarios.


Much of the technical material is relegated to the appendices.  For convenience,
\jmlrcref{sec:notation} summarizes notation, and \jmlrcref{sec:shoulders} contains some 
important supporting results.  Of perhaps practical interest, \jmlrcref{sec:apx_line_search}
provides methods to select the step size, meaning the weight with which new
weak learners are included in the full predictor.  These methods  are sufficiently
powerful to grant the convergence rates in this manuscript.

\subsection{Related Work}
The development of general convergence rates has a number of important milestones in the
past decade. 
\citet{collins_schapire_singer_adaboost_bregman} proved convergence for a large family 
of losses, albeit without any rates.  
Interestingly, the step size only partially modified the choice from AdaBoost to
accommodate arbitrary losses, whereas the choice here follows standard
optimization principles based purely on the particular loss.
Next, \citet{bickel_ritov_zakai} showed a general rate of $\cO(\exp(1/\epsilon^2))$ for
a slightly smaller family of functions: every loss has positive lower and upper bounds on
its second derivative within any compact interval.  This is a larger family than what is
considered in the present manuscript, but \jmlrcref{sec:rate:attainable} will discuss the role of the extra 
assumptions when producing fast rates.

Many extremely important cases have also been handled.  The first
is the original 
rate of $\cO(\ln(1/\epsilon))$ for the exponential loss under the weak learning assumption
\citep{freund_schapire_adaboost}.
Next, under the assumption that the empirical risk minimizer is attainable,
\citet{convergence_attainability} demonstrated the rate $\cO(\ln(1/\epsilon))$.  The loss 
functions in that work must satisfy lower and upper bounds on the Hessian within 
the initial level set; equivalently, the existence of lower and upper bounding
quadratic functions within this level set.  
This assumption may be slightly relaxed to needing just lower and upper second derivative 
bounds on the univariate loss function within an initial bounding
interval (cf. discussion within \jmlrcref{sec:hard_core:attainable}), which is the same set
of assumptions used by \citet{bickel_ritov_zakai}, and as discussed in
\jmlrcref{sec:rate:attainable}, is all that is really needed by the analysis in the
present manuscript under attainability.

Parallel to the present work,
\citet{mukherjee_rudin_schapire_adaboost_convergence_rate} established 
general convergence under the exponential loss, with a rate of $\Theta(1/\epsilon)$.
That work also presented bounds comparing the AdaBoost suboptimality to any $l^1$ bounded solution,
which can be used to succinctly prove consistency properties of AdaBoost \citep{schapire_freund_book}.
In this case, the rate degrades to $\cO(\epsilon^{-5})$, which although presented without
lower bound, is not terribly surprising since the optimization problem minimized by boosting
has no norm penalization.  Finally, mirroring the development here, 
\citet{mukherjee_rudin_schapire_adaboost_convergence_rate} used the same boosting instance
(due to \citet{schapire_adaboost_convergence_rate}) to produce lower bounds, and also 
decomposed the boosting problem into finite and infinite margin pieces (cf. \jmlrcref{sec:hard_core:general}).

It is interesting to mention that, for many variants of boosting, general convergence
rates were known.  Specifically, once it was revealed that boosting is trying to be
not only correct but also have large margins  \citep{boosting_margin}, much work
was invested into methods which explicitly maximized the margin
\citep{warmuth_max_margin_boosting}, or penalized variants focused on the inseparable
case \citep{warmuth_soft_margin,shai_singer_weaklearn_linsep}.  These methods generally
impose some form of regularization~\citep{shai_singer_weaklearn_linsep}, which grants
attainability of the risk minimizer, and allows standard techniques to grant general
convergence rates.  Interestingly, the guarantees in those works cited in this paragraph
are $\cO(1/\epsilon^2)$.


Hints of the dual problem may be found in many works, most notably
those of \citet{kivinen_warmuth_bregman,collins_schapire_singer_adaboost_bregman},
which demonstrated that boosting is seeking a difficult distribution over  training 
examples via iterated Bregman projections.

The notion of hard core sets is due to \citet{russell_hardcore}.  
A crucial difference is that in the present work, the hard core is unique, maximal, and every 
weak learner does no better than random guessing upon a family of 
distributions supported on this set; in this cited work, the hard core is relaxed to allow
some small but constant fraction correlation to the target.
This relaxation is central to the work, which provides a correspondence
between the complexity (circuit
size) of the weak learners, the difficulty of the target function, the size 
of the hard core, and the correlation permitted in the
hard core.

\section{Setup}
\label{sec:setup}
A view of boosting, which pervades this manuscript, is that the action of the weak
learning class upon the sample can be encoded as a matrix 
\citep{convergence_attainability,shai_singer_weaklearn_linsep}.  Let a 
sample $\cS := \{(x_i,y_i)\}_1^m \subseteq (\cX\times \cY)^m$ and a weak learning
class $\cH$ be given.  
For every $h\in \cH$, let $\cS|_h$ denote the negated projection onto
$\cS$ induced by $h$; that is, $\cS|_h$ is a vector of length $m$, with coordinates
$(\cS|_h)_i = -y_ih(x_i)$.  If the set of all such columns $\{\cS|_h : h\in \cH\}$ is finite, 
collect them into the matrix $A\in \R^{m\times n}$.
Let $a_i$ denote the $i^{\textup{th}}$
row of $A$, corresponding to the example $(x_i,y_i)$,  and let $\{h_j\}_1^n$ index
the set of weak learners corresponding to columns of $A$.  It is assumed, for convenience,
that 
entries of $A$ are 
within $[-1,+1]$; relaxing
this assumption merely scales the presented rates by a constant.

The setting considered here is 
that this finite matrix can be constructed.  Note that this can encode infinite classes,
so long as they map to only $k<\infty$ values (in which case $A$ has at most
$k^m$ columns).  As another example, if the weak learners are binary, and $\cH$ has
VC dimension $d$, then Sauer's lemma grants that $A$ has at most $(m+1)^d$ columns. 
This matrix view of boosting is thus similar to the interpretation of boosting performing
descent in functional space
\citep{mason_baxter_bartlett_frean_functionalgrad,friedman_hastie_tibshirani_statboost},
but the class complexity and finite sample have been used
to reduce the function class to a finite object.

To make the connection to boosting, the missing ingredient is the loss function.  
\begin{jdefinition}
$\bG_0$ is the set of loss functions $g:\R\to\R$ satisfying: $g$ is 
twice continuously differentiable, $g'' > 0$, 
and $\lim_{x\to-\infty} g(x) = 0$.

For convenience, whenever $g\in\bG_0$ and sample size $m$ are provided, 
let $f:\R^m \to \R$
denote the empirical risk function
$f(x) := \sum_{i=1}^m g((x)_i)$.  For more properties of $g$ and $f$, please see
\jmlrcref{sec:gfprop}.
\end{jdefinition}

The convergence rates of \jmlrcref{sec:rate} will require a few more conditions, but $\bG_0$ suffices
for all earlier results.

\begin{jexample}
The exponential loss $\exp(\cdot)$ (AdaBoost) and logistic loss $\ln(1+\exp(\cdot))$ are both
within $\bG_0$ (and the eventual $\bG$).  These two losses appear in \jmlrcref{fig:gs}, where the log-scale plot aims
to convey their similarity for negative values.
\end{jexample}

This definition provides a notational break from most boosting literature, which instead
requires $\lim_{x\to\infty} g(x) = 0$ (i.e., the exponential loss becomes $\exp(-x)$);
note that the usage here simply pushes the negation into the definition of the matrix $A$.
The significance of this modification is that the gradient of the empirical risk, which 
corresponds to distributions produced by boosting, is a nonnegative measure. (Otherwise,
it would be necessary to negate this (nonpositive) distribution everywhere to match
the boosting literature.)
Note that there is no consensus on this choice, and the form followed here can be
found elsewhere
\citep{survey_recent_classification_advances}.

\begin{figure}
\begin{center}
\includegraphics[width=0.4\textwidth]{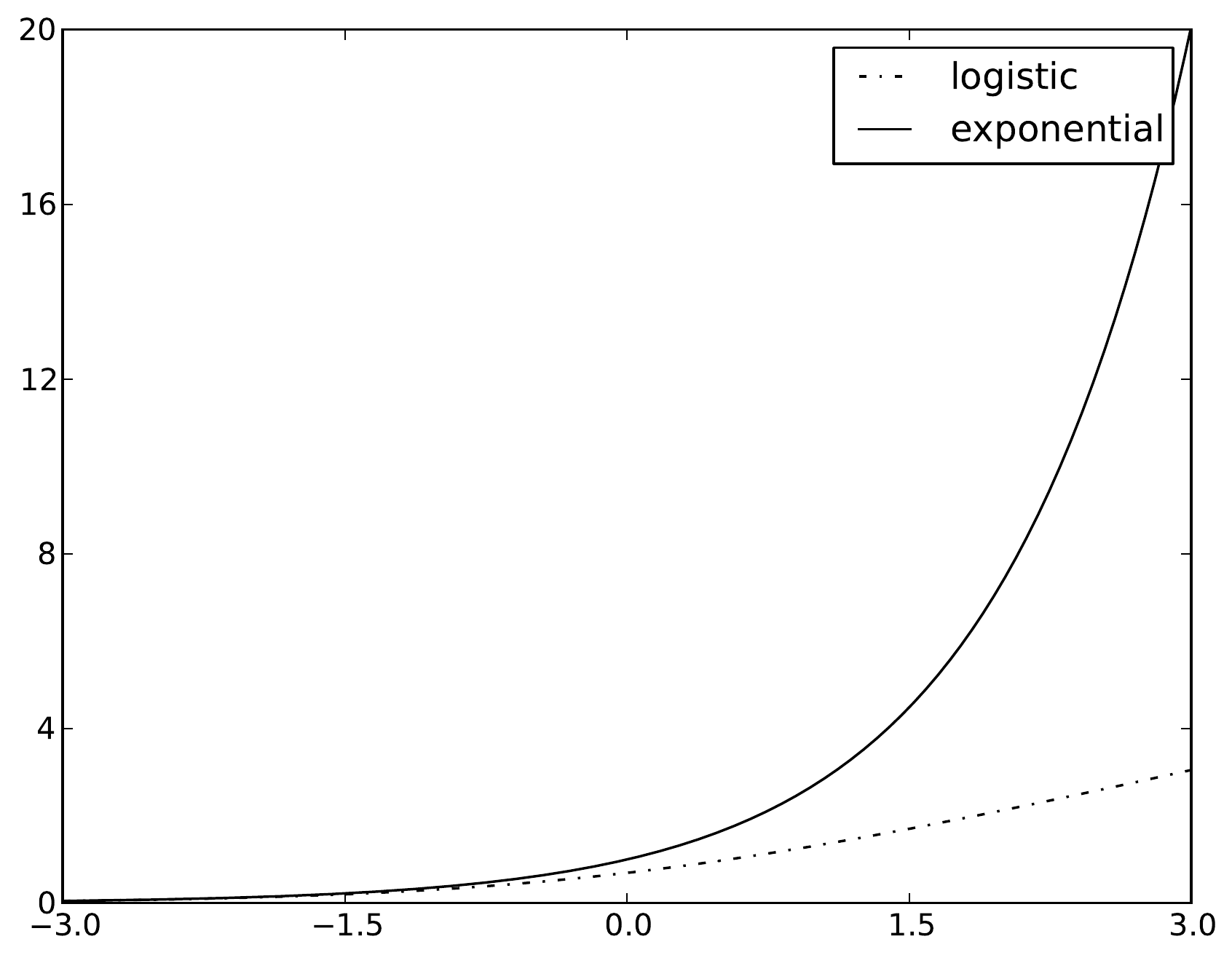}
\includegraphics[width=0.4\textwidth]{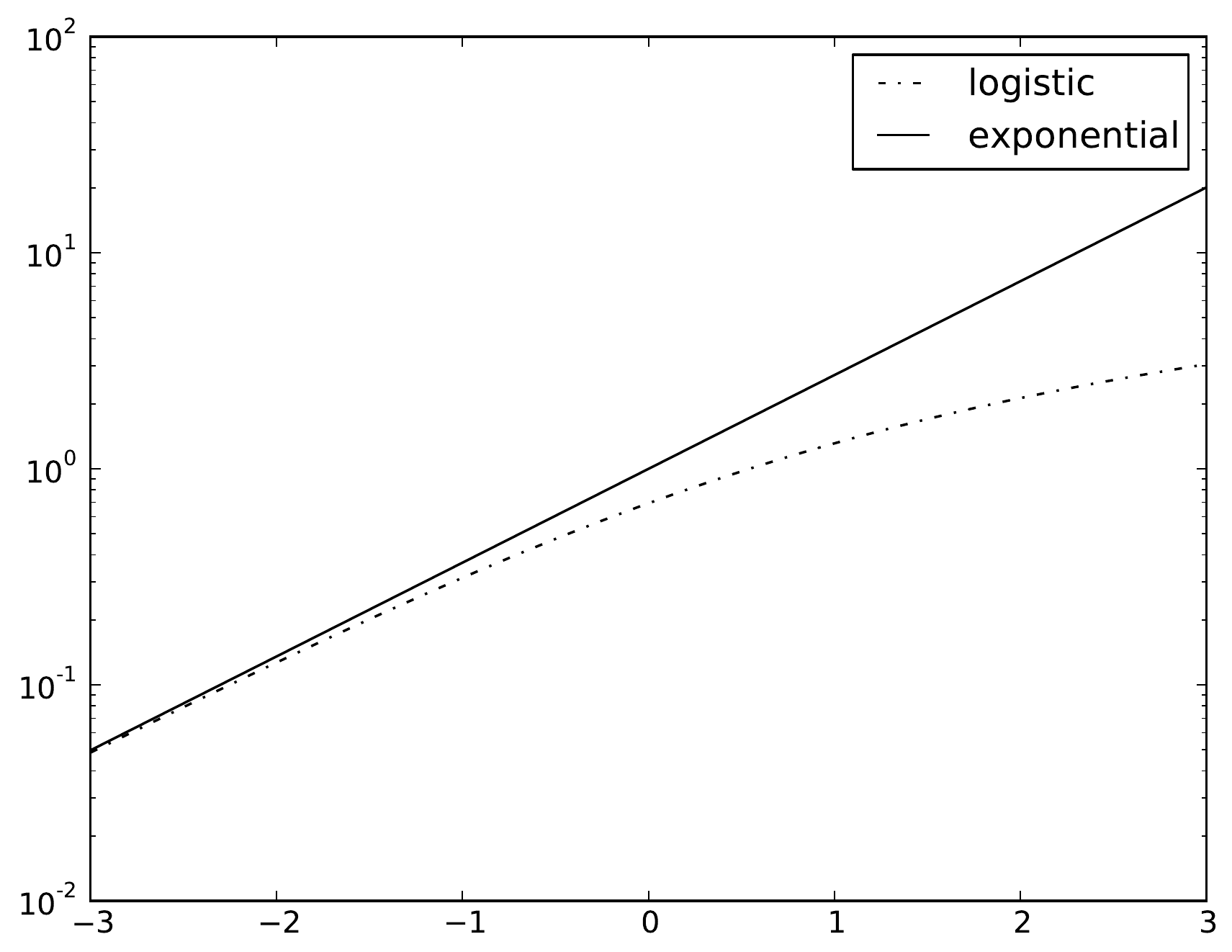}
\end{center}
\caption{Exponential and logistic losses, plotted with linear and log-scale
range.}
\label{fig:gs}
\end{figure}

Boosting determines some weighting $\lambda\in \R^n$ of the columns of $A$, which 
correspond to weak learners in $\cH$.  The (unnormalized) margin of example $i$
is thus $\ip{-a_i}{\lambda} = -\bfe_i^\top A \lambda$, where $\bfe_i$ is an indicator 
vector.  (This negation is one notational inconvenience of making losses increasing.)
Since the prediction on $x_i$ is 
$\1[\sum_j \lambda_j h_j(x_i) \geq 0]= \1[y_i\ip{a_i}{\lambda} \leq 0]$,
it follows that $A\lambda < \bfz_m$ (where 
$\bfz_m$ is the zero vector) implies a training error of zero.  As such, boosting
solves the minimization problem
\begin{equation}
\inf_{\lambda\in \R^n} \sum_{i=1}^m g(\ip{a_i}{\lambda})
=
\inf_{\lambda\in \R^n} \sum_{i=1}^m g(\bfe_i^\top A\lambda) 
=
\inf_{\lambda\in \R^n} f(A\lambda)
=
\inf_{\lambda\in \R^n} (f\circ A)(\lambda)
=:
\bar f_A;
\label{eq:opt}
\end{equation}
recall $f:\R^m \to \R$ is the convenience function $f(x) = \sum_i g((x)_i)$, and in the
present problem denotes the (unnormalized) empirical risk.  $\bar f_A$ will denote
the optimal objective value.

The infimum in \jmlreqref{eq:opt} may well not be attainable.
Suppose there exists $\lambda'$ such that
$A\lambda' < \bfz_m$ (\jmlrcref{fact:gordan:derived} will show that this 
is equivalent to the weak learning
assumption).  Then
\[
0 \leq \inf_{\lambda \in \R^n} f(A\lambda) \leq \inf_{c > 0} f(A(c\lambda')) = 0.
\]
On the other hand, for any $\lambda\in \R^n$, $f(A\lambda) > 0$. 
Thus
the
infimum is never attainable when weak learnability holds.



The template boosting algorithm appears in \jmlrcref{fig:boost}, formulated in terms of
$f\circ A$ to make the connection to coordinate descent as clear as possible.  
To interpret the gradient terms, note that
\[
(\nabla(f\circ A)(\lambda))_j
= (A^\top \nabla f(A\lambda))_j
= -\sum_{i=1}^m g'(\ip{a_i}{\lambda}) h_j(x_i) y_i,
\]
which is the expected negative correlation of $h_j$ with the target labels according to
an unnormalized distribution with weights $g'(\ip{a_i}{\lambda})$.  The stopping
condition $\nabla(f\circ A)(\lambda) = \bfz_m$ means: either the distribution is
degenerate (it is exactly zero), or every weak learner is uncorrelated with the 
target.


As such, $\Boost$ in \jmlrcref{fig:boost}
represents an equivalent formulation of boosting, with one
minor modification: the column (weak learner) selection has an absolute value.
But note that this is the same as closing $\cH$ under complementation (i.e., for 
any $h\in\cH$,
there exists $h^{(-)}$ with $h(x) = -h^{(-)}(x)$), which is assumed in many theoretical
treatments of boosting.  

\begin{figure}
\begin{center}
\framebox[0.9\textwidth][c]{
\begin{minipage}{0.85\textwidth}
\textbf{Routine} \textsc{Boost}.\\
\textbf{Input} Convex function $f\circ A$.\\
\textbf{Output} Approximate primal optimum $\lambda$.

\begin{enumerate}
\item Initialize $\lambda_0 := \bfz_n$.
\item For $t= 1,2,\ldots$, while $\nabla(f \circ A)(\lambda_{t-1}) \neq \bfz_n$:
\begin{enumerate}
\item Choose column (weak learner)
\[
j_t
:= \argmax_{j} | \nabla(f\circ A)(\lambda_{t-1})^\top \bfe_j|.
\]
\item Correspondingly, set descent direction $v_t \in \{\pm \bfe_{j_t}\}$; note
\[
v_t^\top \nabla (f\circ A)(\lambda_{t-1}) = -\|\nabla (f\circ A)(\lambda_{t-1})\|_\infty.
\]
\item Find $\alpha_t$ via approximate solution to the line search 
\[
\inf_{\alpha>0} (f\circ A)(\lambda_{t-1} + \alpha v_t).
\]
\item Update $\lambda_t := \lambda_{t-1} + \alpha_t v_t$.
\end{enumerate}
\item Return $\lambda_{t-1}$.
\end{enumerate}
\end{minipage}
}
\end{center}
\caption{$l^1$ steepest 
\ifjmlr
descent~\citep[Algorithm 9.4]{boyd_vandenberghe}
\else
descent~\citep[Algorithm 9.4]{boyd_vandenberghe}
\fi
of $f\circ A$.}
\label{fig:boost}
\end{figure}

In the case of the exponential loss and binary weak learners,
the line search (when attainable) has a convenient closed
form; but for other losses, and even with the exponential loss but with confidence-rated
predictors, there may not be a closed form.
As such, $\textsc{Boost}$ 
only requires an approximate line search
method.  \Cref{sec:apx_line_search} details two mechanisms for this: an iterative method,
which requires no knowledge of the loss function,
and a closed form choice, which unfortunately requires some properties of
the loss, which may be difficult to bound tightly.  The iterative method provides
a slightly worse guarantee, but is potentially more effective in practice; thus it will be
used to produce all convergence rates in \jmlrcref{sec:rate}.

For simplicity, it is supposed that the best weak learner $j_t$ (or the approximation
thereof encoded in $A$) can always be selected.  Relaxing this condition is not without
subtleties, but as discussed in \jmlrcref{sec:apx_coordinate_selection}, there are ways to
allow approximate selection without degrading the presented convergence rates.

As a final remark, consider the rows $\{-a_i\}_1^m$ of $-A$ as a collection of $m$
points in $\R^n$.  Due to the form of $g$, $\textsc{Boost}$ is therefore searching
for a halfspace, parameterized by a vector $\lambda$, which contains all of these points.
Sometimes such a halfspace may not exist, and $g$ applies a smoothly increasing penalty
to points that are farther and farther outside it.

\section{Dual Problem}
\label{sec:dual}
Applying coordinate descent to \jmlreqref{eq:opt} represents a valid interpretation of 
boosting, in the sense that the resulting algorithm $\Boost$ is equivalent to the
original. However this representation loses the intuitive operation of boosting
as generating distributions where the current predictor is highly erroneous, and requesting
weak learners accurate on these tricky distributions.  The dual problem will capture
this.

In addition to illuminating the structure of boosting, the dual problem also possesses
a major concrete contribution to the optimization behavior, and specifically the
convergence rates:
the dual optimum is always attainable.

The dual problem will make use of Fenchel conjugates~\citep{HULL,borwein_lewis}; for
any function $h$, the conjugate is
\[
h^*(\phi) = \sup_{x\in \dom(h)} \ip{x}{\phi} - h(x).
\]
\begin{jexample}
\label{ex:gconjs}
The exponential loss $\exp(\cdot)$ has Fenchel conjugate
\[
(\exp(\cdot))^*(\phi) = \begin{cases}
\phi\ln(\phi) - \phi & \textup{when } \phi > 0, \\
0 & \textup{when } \phi = 0, \\
\infty & \textup{otherwise}.
\end{cases}
\]
The logistic loss $\ln(1+\exp(\cdot))$ has Fenchel conjugate
\[
(\ln(1+\exp(\cdot)))^*(\phi)
=\begin{cases}
(1-\phi)\ln(1-\phi) + \phi\ln(\phi) & \textup{when } \phi \in (0,1), \\
0 & \textup{when } \phi \in \{0,1\}, \\
\infty & \textup{otherwise}.
\end{cases}
\]
These conjugates are known respectively as the Boltzmann-Shannon and
Fermi-Dirac entropies 
\ifjmlr
\citep[Commentary, Section 3.3]{borwein_lewis}.
\else
\citep[see][closing commentary, Section 3.3]{borwein_lewis}.
\fi
Please see \jmlrcref{fig:gconjs} for a depiction.
\end{jexample}

\begin{figure}
\begin{center}
\includegraphics[width=0.5\textwidth]{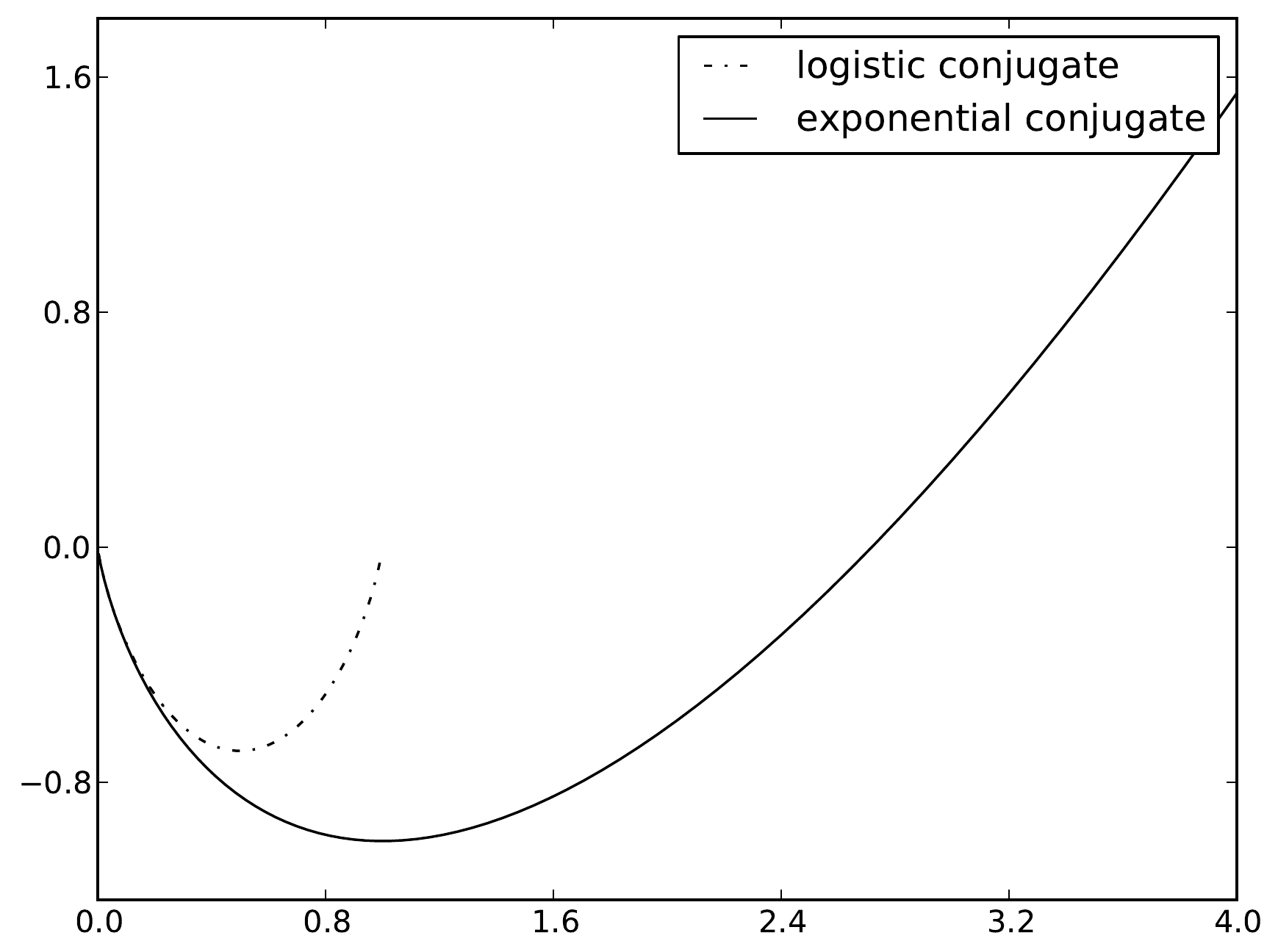}
\end{center}
\caption{Fenchel conjugates of exponential and logistic losses.}
\label{fig:gconjs}
\end{figure}

It further turns out that general members of $\bG_0$ have a shape reminiscent of
these two standard notions of entropy.

\begin{jlemma}
\label{fact:gconj_prop}
Let $g\in \bG_0$ be given.  Then $g^*$ is continuously 
differentiable on $\interior(\dom(g^*))$,
strictly convex, and either
$\dom(g^*) = [0,\infty)$ or $\dom(g^*) = [0,b]$ where $b>0$.  Furthermore, $g^*$
has the following form:
\begin{align*}
g^*(\phi) \in
\begin{cases}
\infty & \textup{when } \phi < 0,\\
0 & \textup{when } \phi = 0,\\
{\phantom{}} (-g(0), 0) & \textup{when } \phi \in (0, g'(0)),\\
-g(0) & \textup{when } \phi = g'(0),\\
(-g(0), \infty] & \textup{when } \phi > g'(0).
\end{cases}
\end{align*}
\end{jlemma}
(The proof is in \jmlrcref{sec:gfprop}.)
There is one more object to present, the dual feasible set $\Phi_A$.

\begin{jdefinition}
For any $A\in\R^{m\times n}$, define the dual feasible set
\[
\Phi_A := \Ker(A^\top) \cap \R^m_+
\qedhere
\]
\end{jdefinition}

Consider any $\psi \in \Phi_A$.  Since $\psi\in\Ker(A^\top)$, this is a weighting
of examples which decorrelates all weak learners from the target: in particular,
for any primal weighting $\lambda\in\R^n$ over weak learners, $\psi^\top A \lambda = 0$.
And since $\psi \in \R^m_+$, all coordinates are nonnegative, so in the case that
$\psi \neq \{\bfz_m\}$, this vector may be renormalized into a distribution over examples.
The case $\Phi_A = \{\bfz_m\}$ is an extremely special degeneracy: it will be shown to
encode the scenario of weak learnability.

\begin{jtheorem}
\label{fact:primal_dual}
For any $A\in \R^{m\times n}$ and $g\in\bG_0$ with $f(x) = \sum_i g((x)_i)$,
\begin{align}
\inf\left\{
f(A\lambda) : \lambda \in \R^n
\right\}
=
\sup\left\{
-f^*(\psi) : \psi \in \Phi_A
\right\}, 
\label{eq:primal_dual}
\end{align}
where
$f^*(\phi) = \sum_{i=1}^m g^*((\phi)_i)$.
The right hand side is the dual problem, and 
moreover 
the dual optimum, denoted $\psi_A^f$, is unique and attainable.
\end{jtheorem}
(The proof uses routine techniques from convex analysis, and is deferred
to \jmlrcref{sec:proof:fact:primal_dual}.)

The definition of $\Phi_A$ does not depend on any specific $g\in \bG_0$; this
choice was made to provide general intuition on the structure of the problem for
the entire family of losses. 
Note however that this will cause some problems later.  For instance,
with the logistic loss, the vector with every value two, i.e. $2\cdot\bfo_m$, has
objective value $-f^*(2\cdot\bfo_m) = -\infty$.  In a sense, there are points in $\Phi_A$ 
which are not really candidates for certain losses, and this fact will need adjustment
in some convergence rate proofs.

\begin{jremark}
Finishing the connection to maximum entropy, for any $g\in\bG_0$, by
\jmlrcref{fact:gconj_prop}, the 
optimum of the unconstrained problem is $g'(0) \bfo_m$, a rescaling of the uniform
distribution.  But note that $\nf(A\lambda_0) = \nf(\bfz_m) = g'(0)\bfo_m$: that is,
the initial dual iterate is the unconstrained optimum!  Let $\phi_t:=\nf(A\lambda_t)$ denote the 
$t^{\textup{th}}$ dual iterate; since $\nf^*(\nf(x)) = x$ (cf.
\jmlrcref{sec:shoulders:fenchel}),
then for any $\psi \in \Phi_A\subseteq \Ker(A^\top)$,
\[
\ip{\nf^*(\phi_t)}{\psi} = \ip{A\lambda_t}{\psi} = \ip{\lambda_t}{A^\top \psi} = 0.
\]
This allows the dual optimum to be rewritten as
\begin{align*}
\psi_A^f &= \argmin_{\psi \in \Phi_A} f^*(\psi)
\\
&= \argmin_{\psi \in \Phi_A} f^*(\psi) - f^*(\phi_t) - \ip{\nf^*(\phi_t)}{\psi - \phi_t};
\end{align*}
that is, the dual optimum $\psi^f_A$  is the Bregman projection (according to $f^*$) onto $\Phi_A$ of any dual
iterate $\phi_t = \nf(A\lambda_t)$.  In particular, $\psi^f_A$ is the Bregman projection onto the feasible 
set of the unconstrained optimum $\phi_0 = \nf(A\lambda_0)$!
\end{jremark}

The connection to Bregman divergences runs deep; in fact, mirroring the development
of $\Boost$ as ``compiling out'' the dual variables in the classical boosting presentation,
it is possible to compile out the primal variables, producing an algorithm using only
dual variables, meaning distributions over examples.  This connection has been
explored extensively~\citep{kivinen_warmuth_bregman,collins_schapire_singer_adaboost_bregman}.

\begin{jremark}
It may be tempting to use \jmlrcref{fact:primal_dual} to produce a stopping condition; that is,
if for a supplied $\epsilon > 0$, a primal iterate $\lambda'$ and dual feasible $\psi' \in
\Phi_A$ can be found satisfying $f(A\lambda') + f^*(\psi') \leq \epsilon$, $\textsc{Boost}$
may terminate with the guarantee $f(A\lambda') - \bar f_A \leq \epsilon$.

Unfortunately, it is unclear how to produce dual iterates (excepting the trivial $\bfz_m$).
If $\Ker(A^\top)$ can be computed, it suffices to 
$l^2$ project $\nf(A\lambda_t)$ onto this subspace. 
In general however, not only is $\Ker(A^\top)$ painfully expensive to compute, this computation does
not at all fit the oracle model of boosting, where access to $A$ is obscured.  (What is $\Ker(A^\top)$
when the weak learning oracle learns a size-bounded decision tree?)

In fact, noting that the primal-dual relationship from \jmlreqref{eq:primal_dual} can be written
\[
\inf \left\{
f(\Lambda) : \Lambda \in \Im(A)
\right\}
=
\sup \left\{
-f^*(\Psi) : \Psi \in \Ker(A^\top) = \Im(A)^\perp
\right\}
\]
(since $\dom(f^*) \subseteq \R^m_+$ encodes the orthant constraint),
the standard oracle model gives elements of $\Im(A)$, but what is needed in the dual is an
oracle for $\Ker(A^\top)=\Im(A)^\perp$.
\end{jremark}

\section{Generalized Weak Learning Rate}
\label{sec:wl}
The weak learning rate
was critical to the original convergence analysis of AdaBoost, providing a handle on 
the progress of the algorithm.  But to be useful, this value must be positive, which
was precisely the condition granted by the weak learning assumption.
This section will generalize the weak learning rate
into a quantity which can be made positive for any boosting 
instance. 

Note briefly that this manuscript will differ slightly from the norm in that
weak learning will be a purely \emph{sample-specific} concept.  That is, the
concern here is convergence in empirical risk, and all that matters is the sample $\cS
= \{(x_i,y_i)\}_1^m$, as encoded in $A$; it doesn't matter if there are wild
points outside this sample, because the algorithm has no access to them.

This distinction has the following implication.  The usual weak learning assumption
states that there exists no uncorrelating distribution over the input \emph{space}.
This of course implies that any training sample $\cS$ used by the algorithm will also
have this property; however, it suffices that there is no distribution over the
input \emph{sample} $\cS$ which uncorrelates the weak learners from the target.

Returning to task, the weak learning assumption posits the existence of a positive 
constant,
the weak learning rate $\gamma$, which lower bounds the correlation of the best weak
learner with the target for any distribution.  Stated in terms of the matrix $A$,
\begin{equation}
0 < \gamma
= \inf_{\substack{\phi\in\R^m_+\\\|\phi\|=1}}
\max_{j\in [n]} \left| \sum_{i=1}^m (\phi)_i y_i h_j(x_i)\right|
= \inf_{\substack{\phi\in\R^m_+\setminus \{\bfz_m\}}} 
\frac {\|A^\top\phi\|_\infty} {\|\phi\|_1}
= \inf_{\substack{\phi\in\R^m_+\setminus \{\bfz_m\}}} 
\frac {\|A^\top\phi\|_\infty} {\|\phi-\bfz_m\|_1}.
\label{eq:wl:classical}
\end{equation}
\begin{jproposition}
\label{fact:wl:Phi_A_point}
A boosting instance is weak learnable iff $\Phi_A = \{\bfz_m\}$.
\end{jproposition}
\begin{proof}
Suppose $\Phi_A = \{\bfz_m\}$; since 
the first infimum in \jmlreqref{eq:wl:classical} is of a continuous function over a compact
set, it has some minimizer $\phi'$.  But $\|\phi'\|_1 = 1$, meaning 
$\phi' \not \in \Phi_A$, and so $\|A^\top\phi'\|_\infty > 0$.  On the other hand,
if $\Phi_A \neq \{\bfz_m\}$, take any $\phi''\in\Phi_A\setminus\{\bfz_m\}$; then
\[
0 \leq \gamma =
\inf_{\substack{\phi\in\R^m_+\setminus \{\bfz_m\}}} 
\frac {\|A^\top\phi\|_\infty} {\|\phi\|_1}
\leq \frac {\|A^\top\phi''\|_\infty}{\|\phi''\|_1} = 0.
\qedhere
\]
\end{proof}

Following this connection, the first way in which the weak learning rate is
modified is to replace $\{\bfz_m\}$ with the dual feasible set $\Phi_A =
\Ker(A^\top) \cap \R^m_+$.  For reasons that will be sketched shortly, but fully dealt
with only in \jmlrcref{sec:rate},
it is necessary to replace $\R^m_+$ with a more refined choice $S$.

\begin{jdefinition}
Given a matrix $A\in \R^{m\times n}$ and a set $S\subseteq \R^m$, define
\[
\gamma(A,S) :=
\inf\left\{
\frac {\|A^\top \phi\|_\infty}{\inf_{\psi \in S\cap \Ker(A^\top)}\|\phi - \psi\|_1}
:\phi \in S \setminus \Ker(A^\top)
\right\}.
\qedhere
\]
\end{jdefinition}

First note that in the scenario of weak learnability (i.e., $\Phi_A = \{\bfz_m\}$ by
\jmlrcref{fact:wl:Phi_A_point}), the choice $S=\R^m_+$ allows the
new notion to exactly cover the old one: $\gamma(A, \R^m_+) = \gamma$.

To get a better handle on the meaning of $S$, 
first define the following projection and distance notation to a closed convex
nonempty set $C$, where in the case of non-uniqueness ($l^1$ and $l^\infty$), some 
arbitrary choice is made:
\begin{align*}
\sfP^p_C(x) \in \Argmin_{y\in C} \|y-x\|_p,
&&
\sfD^p_C(x) = \|x - \sfP^p_C(x)\|_p.
\end{align*}
Suppose, for some $t$, that $\nf(A\lambda_t) \in S\setminus \Ker(A^\top)$;
then the infimum within $\gamma(A,S)$
may be instantiated with $\nf(A\lambda_t)$, yielding
\begin{equation}
\gamma(A,S)
= \inf_{\phi\in S\setminus \Ker(A^\top)}
\frac {\|A^\top\phi\|_\infty} {\|\phi - \sfP^1_{S\cap\Ker(A^\top)}(\phi)\|_1}
\leq 
\frac {\|A^\top\nf(A\lambda_t)\|_\infty}
{\|\nf(A\lambda_t) - \sfP^1_{S\cap \Ker(A^\top)}(\nf(A\lambda_t))\|_1}.
\label{eq:gamma_p:gradient:rearrange:PRE}
\end{equation}
Rearranging this,
\begin{equation}
\gamma(A,S)
\left\|\nf(A\lambda_t) - \sfP^1_{S\cap\Ker(A^\top)}(\nf(A\lambda_t))\right\|_1
\leq 
\|A^\top\nf(A\lambda_t)\|_\infty.
\label{eq:gamma_p:gradient:rearrange}
\end{equation}
This is helpful because the right hand side appears in
standard guarantees for single-step progress in descent methods.  Meanwhile,
the left hand side has reduced the influence of $A$ to a single number, and the
normed expression is the distance to a restriction of dual feasible set,
which will converge to zero if the infimum is to be approached, so long
as this restriction contains the dual optimum.

This will be exactly the approach taken in this manuscript; indeed,
the first step towards convergence rates, \jmlrcref{fact:rate_ub},
will use exactly the upper bound in \eqref{eq:gamma_p:gradient:rearrange}.
The detailed work
that remains is then dealing with the distance to the dual feasible set.
The choice of $S$ will be made to facilitate the production of these bounds,
and will depend on the optimization structure revealed in \jmlrcref{sec:hard_core}.

In order for these expressions to mean anything, $\gamma(A,S)$ must be positive.
\begin{jtheorem}
\label{fact:gamma_p:sanity_check}
Let matrix $A\in \R^{m\times n}$ and polyhedron $S\subseteq \R^m$ be given with
$S\setminus \Ker(A^\top)\neq \emptyset$ and $S\cap \Ker(A^\top)\neq \emptyset$.
Then $\gamma(A,S) > 0$.
\end{jtheorem}
The proof, material on other generalizations of $\gamma$, and discussion
on the polyhedrality of $S$ can all
be found in \jmlrcref{sec:wl:addendum}.

As a final connection,
since $A^\top \sfP^1_{S\cap \Ker(A^\top)}(\phi) = \bfz_n$, note that
\begin{equation*}
\label{eq:gamma_p:projection_form}
\gamma(A,S) 
= \inf_{\phi\in S\setminus \Ker(A^\top)}
\frac {\|A^\top\phi\|_\infty} {\|\phi - \sfP^1_{S\cap\Ker(A^\top)}(\phi)\|_1}
= \inf_{\phi\in S\setminus \Ker(A^\top)}
\frac {\|A^\top(\phi - \sfP^1_{S\cap \Ker(A^\top)}(\phi))\|_\infty}
{\|\phi - \sfP^1_{S\cap\Ker(A^\top)}(\phi)\|_1}.
\end{equation*}
In this way, $\gamma(A,S)$ resembles a Lipschitz constant, reflecting the effect of
$A$ on elements of the dual, relative to the dual feasible set.  

\section{Optimization Structure}
\label{sec:hard_core}
The scenario of weak learnability translates into a simple condition on the dual feasible
set: the dual feasible set is the origin (in symbols, $\Phi_A = \Ker(A^\top)\cap \R^m_+ = \{\bfz_m\}$).
And how about attainability---is there a simple way to encode this problem in terms of 
the optimization problem?

This section will identify the structure of the boosting optimization problem both in terms
of the primal and dual problems, first studying the scenarios of weak learnability and attainability,
and then showing that general instances can be decomposed into these two.

There is another behavior which will emerge through this study, motivated by the following 
question.
The dual feasible set $\Phi_A = \Ker(A^\top)\cap \R^m_+$ is the set of nonnegative weightings
of examples under which every weak learner (every column of $A$) has zero correlation;
what is the support of these weightings? 

\begin{jdefinition}
$H(A)$ denotes the \emph{hard core} of $A$: the collection of examples which receive positive
weight under some dual feasible point, a distribution upon which no weak learner
is correlated with the target.  Symbolically,
\[
H(A) := \{i \in [m] : \exists \psi \in \Phi_A, (\psi)_i > 0\}.
\qedhere
\]
\end{jdefinition}

One case has already been considered; as established in \jmlrcref{fact:wl:Phi_A_point},
weak learnability is equivalent to $\Phi_A = \{\bfz_m\}$, which in turn is equivalent
to $|H(A)| = 0$.  But it will turn out that other possibilities for $H(A)$ also
have direct relevance to the behavior of $\Boost$.
Indeed, contrasted with the primal and dual problems and feasible sets, $H(A)$ will
provide a conceptually simple, discrete object with which to comprehend the behavior
of boosting.

\subsection{Weak Learnability}
The following theorem establishes four equivalent formulations of weak learnability.

\begin{jtheorem}
\label{fact:gordan:derived}
For any $A\in \R^{m\times n}$ and $g\in\bG_0$ the following conditions are equivalent:
\begin{enumerate}[\hspace{1.25cm}(1)\hspace{0.5cm}]
\numberwithin{enumi}{jtheorem}
\item $\exists \lambda\in \R^n \centerdot A\lambda \in \R^m_{--}$,\label{eq:gordan:a}
\item $\inf_{\lambda \in \R^n} f(A\lambda) = 0$,\label{eq:gordan:b}
\item $\psi_A^f = \bfz_m$,\label{eq:gordan:c}
\item $\Phi_A = \{\bfz_m\}$.\label{eq:gordan:d}
\end{enumerate}
\end{jtheorem}

First note that 
\jmlreqref{eq:gordan:d} indicates (via \jmlrcref{fact:wl:Phi_A_point})
this is indeed the weak learnability setting, equivalently
$|H(A)| = 0$.

Recall the earlier discussion of boosting as searching for a halfspace containing the
points $\{-a_i\}_1^m = \{-\bfe_i^\top A\}_1^m$;
property \eqref{eq:gordan:a} encodes precisely this statement, and moreover that there
exists such a halfspace with these points interior to it.  Note that this statement
also encodes the margin separability equivalence of weak learnability due to
\citet{shai_singer_weaklearn_linsep};
specifically, if labels are bounded away from
0 and each point $-a_i$ (row of $-A$) 
is replaced with $-y_ia_i$, the definition of $A$ grants that positive examples will land
on one side of the hyperplane, and negative examples on the other.

The two properties \eqref{eq:gordan:d} and \eqref{eq:gordan:a}
can be interpreted geometrically,
as depicted in \jmlrcref{fig:gordan}:
the dual feasibility statement is that no convex combination of 
$\{-a_i\}_1^m$ will contain the origin.

\begin{figure}[]
\centering
\includegraphics[width=0.35\textwidth]{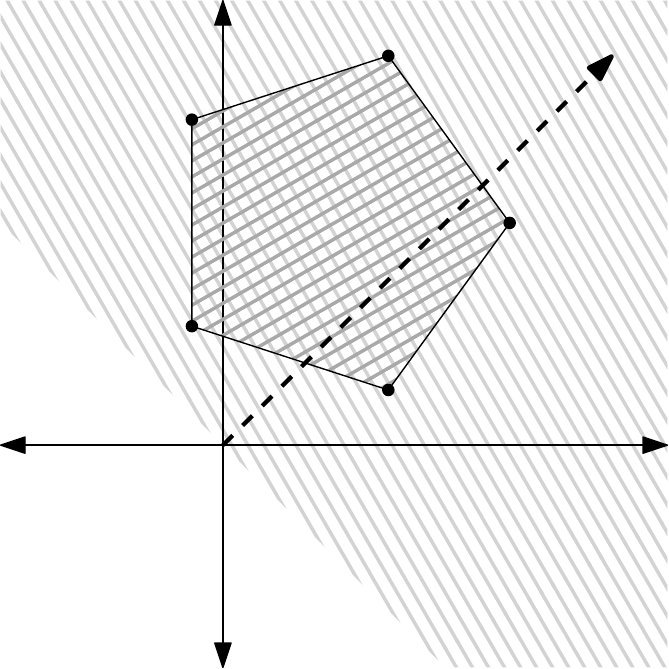}
\caption{Geometric view of the primal and dual problem, under weak learnability.
The vertices of the pentagon denote the points $\{-a_i\}_1^m$.  The arrow, denoting 
$\lambda$ in \eqref{eq:gordan:a}, defines a homogeneous halfspace containing these
points; on the other hand, their convex hull does not contain the origin.  Please
see \jmlrcref{fact:gordan:derived} and its discussion.}
\label{fig:gordan}
\end{figure}

Next, \jmlreqref{eq:gordan:b} is the (error part of the) 
usual strong PAC guarantee~\citep{schapire_wl}: weak
learnability entails that the training error will go to zero.  And, as must be the
case when $\Phi_A = \{\bfz_m\}$, property \jmlreqref{eq:gordan:c} provides that $\psi_A^f = \bfz_m$.

\begin{proof}[Proof of \jmlrcref{fact:gordan:derived}]
($\eqref{eq:gordan:a}\implies\eqref{eq:gordan:b}$)
Let $\bar\lambda\in\R^n$ be given with $A\bar\lambda \in \R^m_{--}$, and let any increasing
sequence $\{c_i\}_1^\infty\uparrow\infty$ be given.  Then, since $f> 0$
and $\lim_{x\to-\infty} g(x) = 0$,
\[
\inf_\lambda f(A\lambda) \leq \lim_{i\to\infty} f(c_iA\bar\lambda) = 0 
\leq \inf_\lambda f(A\lambda).
\]

($\eqref{eq:gordan:b}\implies\eqref{eq:gordan:c}$)
The point $\bfz_m$ is always dual feasible, and
\[
\inf_\lambda f(A\lambda) = 0 = -f^*(\bfz_m).
\]
Since the dual optimum is unique (\jmlrcref{fact:primal_dual}), $\psi_A^f = \bfz_m$.

($\eqref{eq:gordan:c}\implies\eqref{eq:gordan:d}$)
Suppose there exists $\psi \in \Phi_A$ with $\psi \neq \bfz_m$.  Since $-f^*$ is
continuous and increasing
along every positive direction at $\bfz_m = \psi_A^f$ (see \jmlrcref{fact:gconj_prop} and
\jmlrcref{fact:fprop}), there must exist some tiny $\tau > 0$ such that 
$-f^*(\tau\psi) > -f^*(\psi_A^f)$, contradicting the selection of $\psi_A^f$ as the unique
optimum.

($\eqref{eq:gordan:d}\implies\eqref{eq:gordan:a}$)
This case is directly handled by Gordan's theorem (cf. \jmlrcref{fact:gordan}).
\end{proof}

\subsection{Attainability}
\label{sec:hard_core:attainable}
For strictly convex functions, there is a nice characterization of attainability,
which will require the following definition.

\ifjmlr
\begin{jdefinition}[{\citet[Section B.3.2]{HULL}}]
A closed\linebreak[4]
\else
\begin{jdefinition}[{cf. \citet[Definition B.3.2.5]{HULL}}]
A closed
\fi
convex function $h$ is called \emph{0-coercive} when all level sets are compact.
(That is, for any $\alpha\in \R$, the set $\{x : f(x) \leq \alpha\}$ is compact.)
\end{jdefinition}

\begin{jproposition}
\label{fact:strictconvex_0coercive_attainable}
Suppose $h$ is differentiable, strictly convex, and $\dom(h) = \R^m$.
Then $\inf_x h(x)$ is attainable iff $h$ is 0-coercive.
\end{jproposition}
Note that 0-coercivity means the domain of the infimum in \jmlreqref{eq:opt}
can be restricted to a compact set, and attainability in turn follows just from
properties of minimization of continuous functions on compact sets.   It is the
converse which requires some structure; the proof however is unilluminating and
deferred to \jmlrcref{sec:proof:fact:strictconvex_0coercive_attainable}.

Armed with this notion, it is now possible to build an attainability theory for
$f\circ A$.  Some care must be taken with the above concepts, however; note that
while $f$ is strictly convex, $f\circ A$ need not be (for instance, if there
exist nonzero elements of $\Ker(A)$, then moving along these directions does not
change the objective value).  Therefore, 0-coercivity statements will refer to the
function
\[
(f + \iota_{\Im(A)})(x) = \begin{cases}
f(x) &\textup{when } x\in\Im(A), \\
\infty & \textup{otherwise}.
\end{cases}
\]
This function is effectively taking the epigraph of $f$, and intersecting it with a slice
representing $\Im(A)=\{A\lambda : \lambda \in \R^n\}$, the set of points considered by
the algorithm.  As such, it is merely a convenient way of dealing with $\Ker(A)$ as
discussed above.

\begin{jtheorem}
\label{fact:stiemke:derived}
For any $A\in \R^{m\times n}$ and $g\in\bG_0$, the following conditions are equivalent:
\begin{enumerate}[\hspace{1.25cm}(1)\hspace{0.5cm}]
\numberwithin{enumi}{jtheorem}
\item
$\forall \lambda \in \R^n \centerdot A\lambda \not \in \R^m_{-}\setminus\{\bfz_m\}$,
\label{eq:stiemke:a}
\item
$f+\iota_{\Im(A)}$ is 0-coercive,
\label{eq:stiemke:b}
\item
$\psi_A^f \in \R^m_{++}$,
\label{eq:stiemke:c}
\item
$\Phi_A \cap \R^m_{++} \neq \emptyset$.
\label{eq:stiemke:d}
\end{enumerate}
\end{jtheorem}

Following the discussion above, \jmlreqref{eq:stiemke:b} is the desired
attainability statement.

Next, note that \jmlreqref{eq:stiemke:d} is equivalent to the expression $|H(A)| = m$, i.e. there
exists a distribution with positive weight on all examples, upon which every weak learner
is uncorrelated.  The forward direction is direct from the existence of a single
$\psi\in \Phi_A\cap\R^m_{++}$.  For the converse, note that the $\psi_i$ corresponding
to each $i\in H(A)$ can be combined into 
$\psi = \sum_i\psi_i \in \Ker(A^\top)\cap \R^m_{++}$ (since $\Ker(A^\top)$ is a subspace).

For a geometric interpretation, consider \jmlreqref{eq:stiemke:a}
and \jmlreqref{eq:stiemke:d}.  The
first says that any halfspace containing some $-a_i$ within its interior must also
fail to contain some $-a_j$ (with $i\neq j$).  
(Property \jmlreqref{eq:stiemke:a} also allows for the
scenario that no valid enclosing halfspace exists, i.e. $\lambda = \bfz_n$.) 
The latter states that the origin $\bfz_m$
is contained within a positive convex combination of $\{-a_i\}_1^m$ (alternatively,
the origin is within the relative interior of these points).  These two scenarios
appear in \jmlrcref{fig:stiemke}.

\begin{figure}[]
\centering
\includegraphics[width=0.35\textwidth]{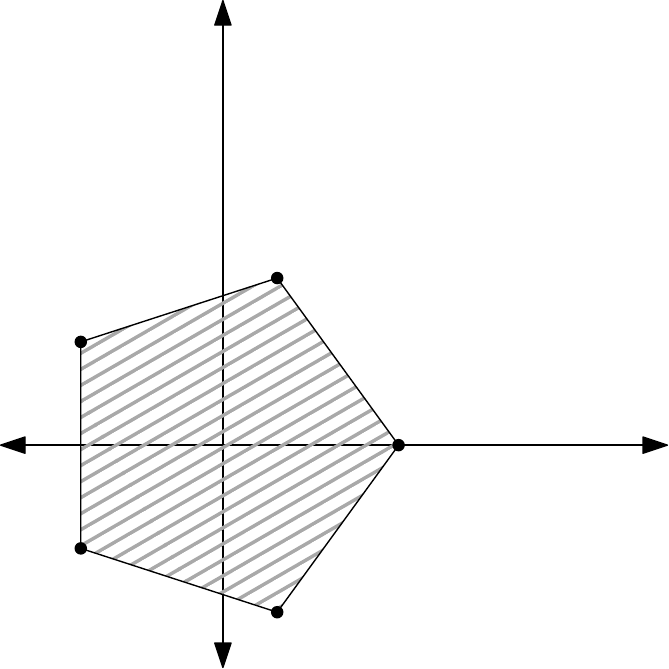}
\caption{Geometric view of the primal and dual problem, under attainability.
Once again, the $\{-a_i\}_1^m$ are the vertices of the pentagon.  This time, no (closed)
homogeneous halfspace 
containing all the points will contain one strictly, and the relative interior of the pentagon
contains the origin.  Please see \jmlrcref{fact:stiemke:derived} and its discussion.}
\label{fig:stiemke}
\end{figure}

Finally, note \jmlreqref{eq:stiemke:c}: 
it is not only the case that there are dual feasible
points fully interior to $\R^m_+$, but furthermore the dual optimum is also interior.
This will be crucial in the convergence rate analysis, since it will allow the dual
iterates to never be too small.

\begin{proof}[Proof of \jmlrcref{fact:stiemke:derived}]
($\eqref{eq:stiemke:a}\implies\eqref{eq:stiemke:b}$)
Let $d\in \R^m\setminus\{\bfz_m\}$ and $\lambda \in \R^n$ be arbitrary.  
To show 0-coercivity, it 
\ifjmlr
suffices~\citep[Proposition B.3.2.4.iii]{HULL}
\else
suffices~\citep[Proposition B.3.2.4.iii]{HULL}
\fi
to show
\begin{align}
\lim_{t\to\infty} \frac{f(A\lambda + td) + \iota_{\Im(A)}(A\lambda+td) - f(A\lambda)}{t} > 0.
\label{eq:stiemke:helperdog:1}
\end{align}
If $d\not\in\Im(A)$ (and $t>0$), then
$\iota_{\Im(A)}(A\lambda + td) = \infty$.  Suppose $d\in\Im(A)$; by
\eqref{eq:stiemke:a}, since $d\neq \bfz_m$, then $d\not\in\R^m_-$, meaning there is at least
one positive coordinate $j$.  But then, since $g > 0$ and $g$ is convex,
\begin{align*}
\eqref{eq:stiemke:helperdog:1}
&\geq \lim_{t\to\infty} \frac{g(\bfe_j^\top(A\lambda +td)) - f(A\lambda)}{t}
\\
&\geq \lim_{t\to\infty} \frac{g(\bfe_j^\top A\lambda) + td_j g'(\bfe_j^\top A\lambda) - f(A\lambda)}{t} 
\\
&= d_j g'(\bfe_j^\top A\lambda),
\end{align*}
which is positive by the selection of $d_j$ and since $g' > 0$.

($\eqref{eq:stiemke:b}\implies\eqref{eq:stiemke:c}$)
Since 
the infimum is attainable, designate any
$\bar\lambda$ satisfying $\inf_\lambda f(A\lambda) = f(A\bar\lambda)$ (note, although
$f$ is strictly convex, $f\circ A$ need not be, thus uniqueness is not guaranteed!).
The optimality conditions of Fenchel problems may be applied,
meaning $\psi_A^f = \nf(A\bar\lambda)$, which is interior to $\R^m_+$ since
$\nf \in\R^m_{++}$ everywhere (cf. \jmlrcref{fact:fprop}).
(For the optimality conditions, see 
\ifjmlr
\citet[Exercise 3.3.9.f]{borwein_lewis},
\else
\citet[Exercise 3.3.9.f]{borwein_lewis},
\fi
with a negation inserted to match the negation inserted within the proof of
\jmlrcref{fact:primal_dual}.)

($\eqref{eq:stiemke:c}\implies\eqref{eq:stiemke:d}$)
This holds since $\Phi_A \supseteq \{\psi_A^f\}$ and $\psi_A^f\in\R^m_{++}$.

($\eqref{eq:stiemke:d}\implies\eqref{eq:stiemke:a}$)
This case is directly handled by Stiemke's Theorem (cf. \jmlrcref{fact:stiemke}).
\end{proof}

\subsection{General Setting}
\label{sec:hard_core:general}
So far, the scenarios of weak learnability and attainability corresponded to the
extremal hard core cases of $|H(A)| \in \{0,m\}$.
The situation in the general setting $1 \leq |H(A)|\leq m-1$ is basically as good
as one could hope for: it interpolates between the two extremal cases.

As a first step, partition $A$ into two submatrices according to $H(A)$.
\begin{jdefinition}
Partition $A\in \R^{m\times n}$ by rows into two matrices $A_0\in\R^{m_0\times n}$
and $A_+\in\R^{m_+\times n}$, where $A_+$ has rows corresponding to $H(A)$, and
$m_+ = |H(A)|$.  For convenience, permute the examples so that 
\[
A = \left[
\begin{smallmatrix}
A_0 \\ A_+
\end{smallmatrix}
\right].
\]
(This merely relabels the coordinate axes, and does not change the optimization problem.)
Note that this decomposition is unique, since $H(A)$ is uniquely specified.
\end{jdefinition}

As a first consequence, this partition cleanly decomposes the dual
feasible set $\Phi_A$ into $\Phi_{A_0}$ and $\Phi_{A_+}$.

\begin{jproposition}
\label{fact:hard_core:Phi_A:decomp}
For any $A\in \R^{m\times n}$, $\Phi_{A_0} = \{\bfz_{m_0}\}$,
$\Phi_{A_+}\cap \R^{m_+}_{++} \neq \emptyset$, and
\[
\Phi_A = \Phi_{A_0} \times \Phi_{A_+}.
\]
Furthermore, no other partition of $A$
into $B_0\in \R^{z\times n}$ and $B_+ \in \R^{p\times n}$  satisfies these properties.
\end{jproposition}

\begin{proof}
It must hold that $\Phi_{A_0} = \{\bfz_{m_0}\}$, since otherwise there would
exist $\psi\in \Ker(A_0^\top)\cap \R^{m_0}_+$ with $\psi \neq \bfz_{m_0}$, which could be extended
to $\psi' = \psi \times \bfz_{m_+} \in \Phi_A$ and the positive coordinate
of $\psi$ could be added to $H(A)$, contradicting the construction of $H(A)$ as including
all such rows.

The property $\Phi_{A_+} \cap \R^{m_+}_{++} \neq \emptyset$ was proved in the discussion
of \jmlrcref{fact:stiemke:derived}: simply add
together, for each $i\in H(A)$, the $\psi_i$'s corresponding to positive weight on $i$.

For the decomposition, note first that certainly every $\psi \in
\Phi_{A_0} \times \Phi_{A_+}$ satisfies
$\psi \in \Phi_A$.  Now suppose contradictorily that there
exists $\psi' \in \Phi_{A} \setminus (\Phi_{A_0}\times \Phi_{A_+})$.  There must
exist $j \in [m] \setminus H(A)$ with $(\psi')_j > 0$, since otherwise
$\psi' \in \{\bfz_z\} \times \Phi_{A_+}$; but that means $j$ should
have been included in $H(A)$, a contradiction.

For the uniqueness property, suppose some other $B_0,B_+$ is given, satisfying the
desired properties.  It is impossible that some $a_i\in B_+$ is not in $H(A)$, 
since any $\psi\in\Phi_{B_+}$ can be extended to $\psi'\in\Phi_A$ with positive weight
on $i$, and thus is included in $H(A)$ by definition.  But the other case with
$i\in H(A)$ but $a_i \in B_0$ is equally untenable, 
since the corresponding measure $\psi_i$ is in $\Phi_A$ but not in
$\Phi_{B_0}\times \Phi_{B_+}$.
\end{proof}

The main result of this section will have the same two main ingredients as
\jmlrcref{fact:hard_core:Phi_A:decomp}:
\begin{itemize}
\item
The full boosting instance may be uniquely decomposed into two pieces, $A_0$ and $A_+$,
each of which individually behave like the weak learnability and attainability scenarios.
\item The subinstances have a somewhat independent effect on the full instance.
\end{itemize}

\begin{jtheorem}
\label{fact:megagordan}
Let $g\in \bG_0$ and $A\in\R^{m\times n}$ be given.
Let $B_0\in \R^{z\times n}$, $B_+\in\R^{p\times n}$ be any partition of $A$ by rows.
The
following conditions are equivalent:
\begin{enumerate}[\hspace{1.25cm}(1)\hspace{0.5cm}]
\numberwithin{enumi}{jtheorem}
\item
$\exists \lambda\in\R^n \centerdot B_0\lambda \in \R^z_{--} \land B_+\lambda = \bfz_p$
$\quad$ and $\quad$
$\forall \lambda\in \R^n \centerdot B_+\lambda \not \in \R^p_- \setminus\{\bfz_p\}$,
\label{eq:decomp:a}
\item
$\inf_{\lambda\in\R^n} f(A\lambda) = \inf_{\lambda\in\R^n} f(B_+\lambda)$,
$\quad$ and $\quad$ $\inf_{\lambda\in\R^n} f(B_0\lambda) = 0$,
\\
and $\quad$ $f + \iota_{\Im(B_+)}$ is 0-coercive,
\label{eq:decomp:b}
\item
$\psi_A^f = \left[\begin{smallmatrix}\psi_{B_0}^f\\\psi_{B_+}^f\end{smallmatrix}\right]$
$\quad$ with $\quad$
$\psi_{B_0}^f = \bfz_z$
$\quad$ and $\quad$ $\psi_{B_+}^f\in \R^p_{++}$,
\label{eq:decomp:c}
\item
$\Phi_{B_0} = \{\bfz_z\}$, 
$\quad$ and $\quad$ $\Phi_{B_+}\cap \R^p_{++} \neq \emptyset$,
$\quad$ and $\quad$
$\Phi_A = \Phi_{B_0} \times \Phi_{B_+}$.
\label{eq:decomp:d}
\end{enumerate}
\end{jtheorem}

Stepping through these properties, notice that \jmlreqref{eq:decomp:d} mirrors the 
expression in \jmlrcref{fact:hard_core:Phi_A:decomp}.  But that 
\namecref{fact:hard_core:Phi_A:decomp} also granted that this representation was
unique, thus only one partition of $A$ satisfies the above properties, namely $A_0,A_+$.
Since this \namecref{fact:megagordan} is stated as a series of equivalences, any one
of these properties can in turn be used to identify the hard core set $H(A)$.

To continue with geometric interpretations, notice that \jmlreqref{eq:decomp:a} states that
there exists a halfspace strictly containing those points in $[m]\setminus H(A)$, with all
points of $H(A)$ on its boundary; furthermore, trying to adjust this halfspace to
contain elements of $H(A)$ will place others outside it.  With regards to the geometry
of the dual feasible set as provided by \jmlreqref{eq:decomp:d},
the origin is within the relative interior of the points 
corresponding to $H(A)$, however the convex hull of the other $m  - |H(A)|$ points
can not contain the origin.  Furthermore, if the origin is written as a convex
combination of all points, this combination must place zero weight
on the points with indices $[m]\setminus H(A)$.
This scenario is depicted in \jmlrcref{fig:megagordan}.

\begin{figure}[]
\centering
\includegraphics[width=0.35\textwidth]{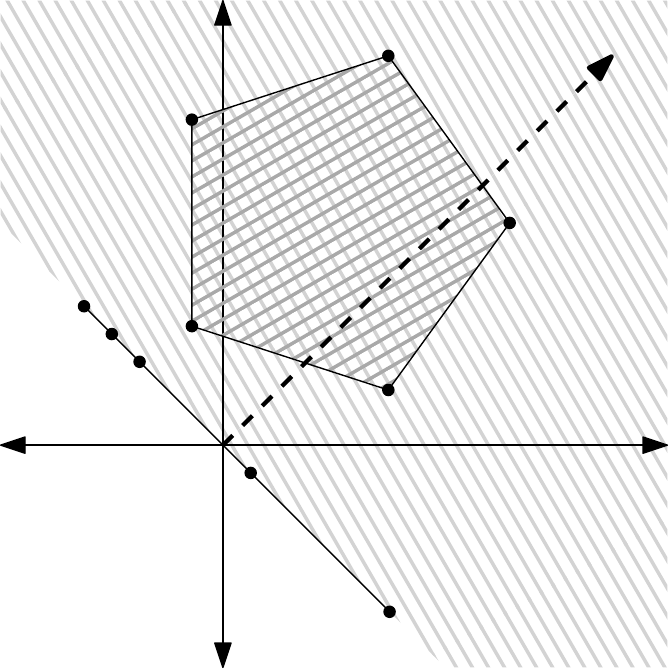}
\caption{Geometric view of the primal and dual problem in the general case.
There is a closed homogeneous halfspace containing the points $\{-a_i\}_1^m$,
where the hard core lies on the halfspace boundary, and the other points are within
its interior; moreover, there does not exist a closed homogeneous halfspace containing
all points but with strict containment on a point in the hard core.
Finally, although the origin is in the convex hull of $\{-a_i\}_1^m$, any such convex
combination places zero weight on points outside the hard core.
Please see \jmlrcref{fact:megagordan} and its discussion.}
\label{fig:megagordan}
\end{figure}


In properties \eqref{eq:decomp:b} and \eqref{eq:decomp:c}, $B_0$ mirrors the behavior of
weakly learnable instances in \jmlrcref{fact:gordan:derived}, and analogously $B_+$
follows instances with minimizers from \jmlrcref{fact:stiemke:derived}.  The interesting
addition, as discussed above, is the independence of these components: 
\jmlreqref{eq:decomp:b} provides that the infimum of the combined problem is the sum
of the infima of the subproblems, while \jmlreqref{eq:decomp:c} provides that the full
dual optimum may be obtained by concatenating the subproblems' dual optima.

\begin{proof}[Proof of \jmlrcref{fact:megagordan}]
($\eqref{eq:decomp:a}\implies\eqref{eq:decomp:b}$)  Let $\bar\lambda$ be given
with $B_0\bar\lambda\in \R^z_{--}$ and $B_+\bar\lambda = \bfz_p$, and let 
$\{c_i\}_1^\infty
\uparrow \infty$ be an arbitrary sequence increasing without bound.  
Lastly, let $\{\lambda_i\}_1^\infty$ be a minimizing sequence for 
$\inf_\lambda f(B_+\lambda)$.  
Then
\begin{align*}
\inf_\lambda f(B_+\lambda)
&= \lim_{i\to\infty}\left( f(B_+\lambda_i) + f(c_iB_0\bar\lambda)\right)
\geq \inf_\lambda f(A\lambda) 
\\
&
= \inf_\lambda (f(B_+\lambda) + f(B_0\lambda)) 
\geq \inf_\lambda f(B_+\lambda),
\end{align*}
which used the fact that $f(B_0\lambda)\geq 0$ since $f\geq 0$.  And since the chain 
of inequalities starts and ends the same, it must be a chain of equalities, which
means $\inf_\lambda f(B_0\lambda) = 0$. 
To show 0-coercivity of $f+\iota_{\Im(B_+)}$,
note the second part of $\eqref{eq:decomp:a}$ is one of the 
conditions of \jmlrcref{fact:stiemke:derived}.

($\eqref{eq:decomp:b}\implies\eqref{eq:decomp:c}$)  First,
by \jmlrcref{fact:gordan:derived}, $\inf_\lambda f(B_0\lambda) = 0$ 
means $\psi_{B_0}^f = \bfz_z$ and $\Phi_{B_0} = \{\bfz_z\}$.   Thus
\begin{align*}
-f^*(\psi_{A}^f) 
&= \sup_{\psi \in \Phi_A} -f^*(\psi) \\
&= 
\sup\left\{
-f^*(\psi_z) -f^*(\psi_p)
:
\psi_z\in \R^z_+, \psi_p\in \R^p_+, B_0^\top \psi_z + B_+^\top\psi_p = \bfz_n\right\}\\
&\geq 
\sup_{\psi_z\in \Phi_{B_0}} -f^*(\psi_z) 
+\sup_{\psi_p\in \Phi_{B_+}} -f^*(\psi_p)  \\
&= 0 -f^*(\psi_{B_+}^f)
= \inf_{\lambda \in \R^n} f(B_+\lambda) 
= \inf_{\lambda \in \R^n} f(A\lambda) 
= -f^*(\psi_A^f).
\end{align*}
Combining this with $f^*(x) =\sum_i g((x)_i)$ and $g^*(0) = 0$ (cf. \jmlrcref{fact:gconj_prop,fact:primal_dual}),
$f^*(\psi_A^f) = f^*(\psi_{B_+}^f) = f^*(\left[
\begin{smallmatrix}\psi_{B_0}^f\\\psi_{B_+}^f\end{smallmatrix}
\right])$.
But \jmlrcref{fact:primal_dual} shows $\psi_A^f$ was unique, which gives the result.
And to obtain $\psi_{B_+}^f\in \R^p_{++}$, use \jmlrcref{fact:stiemke:derived} with
the 0-coercivity of $f + \iota_{\Im(B_+)}$.

($\eqref{eq:decomp:c}\implies\eqref{eq:decomp:d}$) 
Since $\psi_{B_0}^f = \bfz_z$, it follows by \jmlrcref{fact:gordan:derived}
that $\Phi_{B_0} = \{\bfz_z\}$.  Furthermore, since $\psi_{B_+}^f \in \R^p_{++}$,
it follows that $\Phi_{B_+} \cap \R^p_{++} \neq \emptyset$.  Now
suppose contradictorily that $\Phi_A \neq \Phi_{B_0}\times \Phi_{B_+}$; since it
always holds that $\Phi_A \supseteq \Phi_{B_0} \times \Phi_{B_+}$, this supposition
grants the existence of 
$\psi = \left[\begin{smallmatrix}\psi_z\\\psi_p\end{smallmatrix}\right]\in \Phi_A$ where
$\psi_z\in\R^z_+\setminus\{\bfz_z\}$.

Consider the element $q:=\psi + \psi_A^f$, which
has more nonzero entries than $\psi_A^f$, but still $q\in\Phi_A$
since $\Phi_A$ is a convex cone.
Let $I_q$ index the nonzero entries of $q$,
and let $A_q$ be the restriction of $A$ to the rows $I_q$.  Since $q\in \Phi_A$,
meaning $q$ is nonnegative and $q\in \Ker(A^\top)$, it follows that
the restriction of $q$ to its positive entries is within
$\Ker(A_q^\top)$ (because only zeros of $q$ and matching rows of $A$ 
are removed, dot products between $q$ 
with rows of $A^\top$ are the same as dot products between the restriction of $q$
and rows of $A_q^\top$),
and so $q\in \Phi_{A_q}$, 
meaning $\Phi_{A_q} \cap \R^{|I_q|}_{++}$ is nonempty.
Correspondingly, by \jmlrcref{fact:stiemke:derived}, the dual optimum
$\psi_{A_q}^f$ of this restricted problem will have only positive entries.
But by the same reasoning granting that $q$ restricted to $I_q$ is within $\Phi_{A_q}$,
it follows that the full optimum $\psi_A^f$, restricted to $I_q$, must also
be within $\Phi_{A_q}$ (since, by $q$'s construction,
$\psi_A^f$'s zero entries are a superset of the zero entries
of $q$).  Therefore this restriction $\hat \psi_A^f$ of $\psi_A^f$ 
to $I_q$ will have at least one zero entry,
meaning it can not be  equal to
$\psi_{A_q}^f$; but 
\jmlrcref{fact:primal_dual} provided that the dual optimum is unique,
thus $-f^*(\psi_{A_q}^f) > -f^*(\hat\psi_{A}^f)$.
Finally, produce $\bar \psi_{A_q}^f$ from $\psi_{A_q}^f$ by
inserting a zero for each entry of $I_q$;
the same reasoning that allows feasibility to be maintained while removing
zeros allows them to be added, and thus $\bar \psi_{A_q}^f\in \Phi_A$.  
But this is a contradiction: since $g^*(0) = 0$ (cf. \jmlrcref{fact:gconj_prop}),
both $\bar\psi_{A_q}^f$
and the optimum $\psi_A^f$ have zero contribution to the
objective along the entries outside of $I_q$, and thus
\[
-f^*(\bar\psi_{A_q}^f)
= -f^*(\psi_{A_q}^f)
> -f^*(\hat \psi_A^f)
= -f^*(\psi_A^f),
\]
meaning $\bar\psi_{A_q}^f$ is feasible and has strictly greater objective value
than the optimum $\psi_A^f$, a contradiction.





($\eqref{eq:decomp:d}\implies\eqref{eq:decomp:a}$) 
Unwrapping the definition of $\Phi_A$, the assumed statements imply
\[
(\forall \phi_0\in \R^z_+\setminus \{\bfz_z\},\phi_+\in \R^p_+
\centerdot B_0^\top \phi_0 + B_+^\top\phi_+ \neq \bfz_n)
\land (\exists \phi_+\in\R^p_{++}\centerdot B_+^\top \phi_+ = \bfz_n).
\]
Applying Motzkin's transposition theorem (cf. \jmlrcref{fact:motzkin}) to the left
statement and Stiemke's theorem (cf. \jmlrcref{fact:stiemke}, which is implied by
Motzkin's theorem) to the right yields
\[
(\exists \lambda\in\R^n \centerdot B_0\lambda \in \R^z_{--} \land B_+\lambda \in \R^p_-)
\land (\forall \lambda\in \R^n \centerdot B_+\lambda \not \in \R^p_- \setminus\{\bfz_p\}),
\]
which implies the desired statement.
\end{proof}

\begin{jremark}
Notice the dominant role $A$ plays in the structure of the solution found by
boosting.  For every $i\in [m] \setminus H(A)$, the corresponding dual weights 
go
to zero (i.e., $(\nf(A\lambda_t))_i\downarrow 0$),
and the corresponding primal margins 
grow unboundedly
(i.e., \nolinebreak[4]{$-\bfe_i^\top A\lambda_t \uparrow \infty$}, since otherwise $\inf_\lambda f(A_0\lambda) > 0$).
This is completely unaffected by the choice
of $g\in \bG_0$.  Furthermore, whether this instance is weak learnable, attainable, or neither
is dictated purely by $A$ (respectively $|H(A)| = 0$, $|H(A)| = m$, or
$|H(A)| \in [1,m-1]$).

Where different loss functions disagree is how they assign dual weight to the points
in $H(A)$.  In particular, each $g\in \bG_0$ (and corresponding $f$) defines a notion
of entropy via $f^*$.  The dual optimization in \jmlrcref{fact:primal_dual} can then
be interpreted as selecting the max entropy choice (per $f^*$) amongst those convex
combinations of $H(A)$ equal to the origin.
\end{jremark}

\section{Convergence Rates}
\label{sec:rate}
Convergence rates will be proved for the following family of loss functions.
\begin{jdefinition}
$\bG$ contains all functions $g$ satisfying the following properties. 
First, $g\in\bG_0$.
Second, for any $x\in \R^m$ satisfying $f(x) \leq f(A\lambda_0) = mg(0)$, 
and for any coordinate $(x)_i$, there exist
constants $\eta> 0$ and $\beta > 0$ such that $g''((x)_i) \leq \eta g((x)_i)$ and 
$g((x)_i) \leq \beta g'((x)_i)$.
\end{jdefinition}

The exponential loss is in this family with $\eta = \beta = 1$ since $\exp(\cdot)$ is a
fixed point with respect to the differentiation operator.
Furthermore, as is verified in \jmlrcref{rem:logistic_loss_in_bG},
the logistic loss is also in this family, with
$\eta = 2^m/(m\ln(2))$ and $\beta = 1+2^m$ (which may be loose).  In a sense, $\eta$ and $\beta$ encode how
similar some $g\in\bG$ is to the exponential loss, and thus these parameters can degrade
radically.  However, outside the weak learnability case, the other terms in the bounds
here 
can also incur a large penalty with the exponential loss,
and there is
some evidence that this is unavoidable (see the lower bounds in 
\citet{mukherjee_rudin_schapire_adaboost_convergence_rate} or the upper bounds
in \citet{convergence_attainability}).

The first step towards proving convergence rates will be to lower bound the improvement
due to one iteration.  As discussed previously,
standard techniques for analyzing descent methods provide such
bounds in terms of gradients, however to overcome the difficulty of unattainability in
the primal space, the key will be to convert this into distances in the dual
via $\gamma(A,S)$, as in 
\jmlreqref{eq:gamma_p:gradient:rearrange}.


\begin{jproposition}
\label{fact:rate_ub}
For any $t$, $g\in \bG$,  $A\in \R^{m\times n}$,   and $S\supseteq \{\nf(A\lambda_t)\}$
with $\gamma(A,S)>0$,
\[
f(A\lambda_{t+1}) - \bar f_A \leq f(A\lambda_t) - \bar f_A -
\frac {\gamma(A,S)^2 \dist^1_{S\cap\Ker(A^\top)}(\nf(A\lambda_t))^2}{6\eta f(A\lambda_t)}.
\]
\end{jproposition}
\begin{proof}
The stopping condition grants $\nf(A\lambda_t) \not \in \Ker(A^\top)$.
Proceeding as in 
\jmlreqref{eq:gamma_p:gradient:rearrange:PRE},
\[
\gamma(A,S)
= \inf_{\phi \in S \setminus \Ker(A^\top)}
\frac {\|A^\top \phi\|_\infty} {\dist^1_{S \cap\Ker(A^\top)}(\phi)}
\leq
\frac {\|A^\top \nf(A\lambda_t)\|_\infty}
{\dist^1_{S\cap\Ker(A^\top)}(\nf(A\lambda_t))}.
\]
Combined with the approximate line search guarantee from \jmlrcref{fact:wolfe_guarantee},
\begin{align*}
f(A\lambda_t) - f(A\lambda_{t+1})
&\geq
\frac {\|A^\top \nf(A\lambda_t)\|_\infty^2} {6\eta f(A\lambda_t)}
\geq
\frac 
{\gamma(A,S)^2 \dist^1_{S \cap \Ker(A^\top)}(\nf(A\lambda_t))^2}
{6\eta f(A\lambda_t)}.
\end{align*}
Subtracting $\bar f_A$ from both sides and rearranging yields the statement.
\end{proof}

The task now is to manage the dual distance
$\dist^1_{S\cap \Ker(A^\top)}(\nf(A\lambda_t))$,
specifically to produce a relation to $f(A\lambda_t) - \bar f_A$, the total suboptimality
in the preceding iteration; from there, standard tools in convex optimization
will yield convergence rates.  
Matching the problem structure revealed in \jmlrcref{sec:hard_core}, first the extremal
cases of weak learnability and attainability will be handled, and only then the
general case.  The significance of this division is that the extremal cases have
rate $\cO(\ln(1/\epsilon))$, whereas the general case has rate $\cO(1/\epsilon)$ (with
a matching lower bound provided for the logistic loss).  The reason, which will be 
elaborated in further sections, is straightforward: the extremal cases are fast for 
essentially opposing regions, and this conflict will degrade the rate in the 
general case.

\subsection{Weak Learnability}
\label{sec:rate:wl}
\begin{jtheorem}
\label{fact:rate:gordan_ada}
Suppose 
$|H(A)| = 0$
and $g\in \bG$; then $\gamma(A,\R^m_+) > 0$, and
for any $t\geq 0$,
\[
f(A\lambda_t) 
\leq f(A\lambda_0)\left(
1 - 
\frac {\gamma(A,\R^m_+)^2} {6\beta^2\eta}
\right)^t.
\]
\end{jtheorem}
\begin{proof}
By \jmlrcref{fact:gordan:derived}, $\Phi_A = \{\bfz_m\}$, meaning
\[
\dist^1_{\Phi_A}(\nf(A\lambda_t))
= \inf_{\psi \in \Phi_A} \|\nf(A\lambda_t) - \psi\|_1
= \|\nf(A\lambda_t)\|_1
\geq f(A\lambda_t) / \beta.
\]
Next, $\R^m_+$ is polyhedral, and \jmlrcref{fact:gordan:derived} grants
$\R^m_+ \cap \Ker(A^\top) \neq \emptyset$ and $\R^m_+ \setminus \Ker(A^\top)\neq 
\emptyset$, so \jmlrcref{fact:gamma_p:sanity_check} provides 
$\gamma(A,\R^m_+) > 0$.
Since $\nf(A\lambda_t)\in \R^m_+$, all conditions of
\jmlrcref{fact:rate_ub} are met, 
and using $\bar f_A = 0$ (again by \jmlrcref{fact:gordan:derived}),
\begin{equation}
f(A\lambda_{t+1})  \leq f(A\lambda_t) -
\frac { \gamma(A,\R^m_+)^2f(A\lambda_t)^2} {6\beta^2\eta f(A\lambda_t)}
= f(A\lambda_t)\left(
1 - 
\frac { \gamma(A,\R^m_+)^2} {6\beta^2\eta}
\right),
\label{eq:rate:gordan_ada:key}
\end{equation}
and recursively applying this inequality yields the result.
\end{proof}

As discussed in \jmlrcref{sec:wl}, $\gamma(A,\R^m_+) = \gamma$,
the latter quantity being the classical weak
learning rate. 

Specializing this analysis to the exponential loss (where $\eta=\beta=1$),
the bound
becomes $(1-\gamma^2/6)^t$, which recovers the bound of
\cite{schapire_singer_confidence_rated}, although with vastly different analysis.
(The exact expression has denominator 2 rather than 6, which can be recovered with
the closed form line search; cf. \jmlrcref{sec:apx_line_search}.)

In general, solving for $t$ in the expression
\[
\epsilon = \frac {f(A\lambda_t) - \bar f_A}{f(A\lambda_0) - \bar f_A}
\leq \left(
1 - 
\frac { \gamma^2} {6\beta^2\eta}
\right)^t
\leq \exp\left(
 - 
\frac { t\gamma^2} {6\beta^2\eta}
\right)
\]
reveals that $t \leq \frac{6\beta^2\eta} {\gamma^2} \ln (1/\epsilon)$
iterations suffice to reach suboptimality
$\epsilon$.  Recall that $\beta$ and $\eta$, in the case of the logistic loss, have only
been bounded by quantities like $2^m$.  While it is unclear if this analysis of $\beta$
and $\eta$ was tight, note that it is plausible that the logistic loss is slower than
the exponential loss in this scenario, as it works less in initial phases to correct
minor margin violations.

\begin{jremark}
The rate $\cO(\ln(1/\epsilon))$ depended crucially on both
$g \leq \beta g'$ and $g'' \leq \eta g$.  If for instance the second inequality were
replaced with $g'' \leq C$, then \jmlreqref{eq:rate:gordan_ada:key} would instead have 
form $f(A\lambda_{t+1}) \leq f(A\lambda_t) - f(A\lambda_t)^2\cO(1)$, which by
an application of \jmlrcref{fact:sss:1} would grant a rate $\cO(1/\epsilon)$.  
For functions which asymptote to zero (i.e., everything in $\bG_0$), satisfying this milder
second order condition is quite easy.  
The real mechanism behind producing a fast rate is 
$g \leq \beta g'$,
which guarantees that
the flattening of the objective function is concomitant with low objective values.
\end{jremark}

\subsection{Attainability}
\label{sec:rate:attainable}
Consider now the case of attainability.  Recall from \jmlrcref{fact:stiemke:derived} and
\jmlrcref{fact:strictconvex_0coercive_attainable} that attainability occurred along with
a stronger property, the 0-coercivity (compact level sets) of $f+ \iota_{\Im(A)}$ (it
was not possible to work with $f\circ A$ directly, which will have unbounded level sets
when $\Ker(A) \neq \bfz_n$).

This has an immediate consequence to the task of relating 
$f(A\lambda_t) - \bar f_A$ to 
the dual distance $\dist^1_{S\cap \Ker(A^\top)}(\nf(A\lambda_t))$.
$f$ is a strictly convex function, which means it
is strongly convex over any compact set.  Strong convexity in the primal corresponds to
upper bounds on second derivatives (occasionally termed \emph{strong smoothness})
in the dual, which in turn can be used to relate distance and objective values.
This also provides the choice of polyhedron $S$ in $\gamma(A,S)$:
unlike the case of weak learnability, 
where the unbounded set $\R^m_+$ was used, a compact set containing the initial
level set will be chosen.

\begin{jtheorem}
\label{fact:rate:stiemke}
Suppose $|H(A)| = m$ and  $g\in\bG$.
Then there exists a (compact) tightest axis-aligned 
rectangle $\cC$ 
containing the initial level set $\{x\in\R^m : (f+\iota_{\Im(A)})(x) \leq f(A\lambda_0)\}$,
and $f$ is strongly convex with modulus $c>0$ over $\cC$.
Finally, either $\lambda_0$ is optimal, 
or $\gamma(A,\nf(\cC)) > 0$, and for all $t$,
\[
f(A\lambda_t) - \bar f_A
\leq (f(A\lambda_0) - \bar f_A)
\left(1-\frac {c\gamma(A,\nf(\cC))^2}{3\eta f(A\lambda_0)}\right)^t.
\]
\end{jtheorem}

As in \jmlrcref{sec:rate:wl}, when $\lambda_0$ is suboptimal, 
this bound may be rearranged to say that 
$t \leq \frac {3\eta f(A\lambda_0)}{c\gamma(A,\nf(\cC))^2} \ln(1/\epsilon)$
iterations suffice
to reach suboptimality $\epsilon$.

To make sense of this bound and its proof, the essential object is $\cC$, whose
properties are captured in the following \namecref{fact:attain:hypercube}, which
is stated with some slight generality in order to allow reuse 
in \jmlrcref{sec:rate:general}.

\begin{jlemma}
\label{fact:attain:hypercube}
Let $g\in\bG$, $A\in\R^{m\times n}$ with $|H(A)| = m$,
and any $d\geq\inf_\lambda f(A\lambda)$ be given.
Then there exists a (compact nonempty) tightest axis-aligned rectangle
$\cC \supseteq \{x\in \R^m : (f+\iota_{\Im(A)})(x) \leq d\}$.
Furthermore, the dual image $\nf(\cC) \subset \R^m$ 
is also a (compact nonempty) axis-aligned rectangle,
and moreover it is strictly contained within
$\dom(f^*) \subseteq \R^m_+$.
Finally, $\nf(\cC)$ contains dual feasible points 
(i.e., $\nf(\cC)\cap \Phi_A \neq \emptyset$).
\end{jlemma}
A full proof may be found in \jmlrcref{sec:proof:fact:attain:hypercube};
the principle is that $|H(A)|=m$ provides 0-coercivity of $f+\iota_{\Im(A)}$, and
thus the initial level set is compact.  To later show $\gamma(A,S)>0$ via
\jmlrcref{fact:gamma_p:sanity_check}, $S$ must be polyhedral, and to
apply \jmlrcref{fact:rate_ub}, it must contain the dual iterates
$\{\nf(A\lambda_t)\}_{t=1}^\infty$; the easiest
choice then is to take the bounding box $\cC$ of the initial level set,
and use its dual map $\nf(\cC)$.  To exhibit dual feasible points within $\nf(\cC)$,
note that $\cC$ will contain a primal minimizer, and optimality conditions grant that
$\nf(\cC)$ contains the dual optimum.

With the polyhedron in place, \jmlrcref{fact:rate_ub} may be applied, so what remains
is to control the dual distance.  Again, this result will be stated with some extra
generality in order to allow reuse in \jmlrcref{sec:rate:general}.

\begin{jlemma}
\label{fact:attain:2ndorder_bound}
Let $A\in \R^{m\times n}$, $g\in\bG$, and
any compact set $S$ with $\nf(S)\cap\Ker(A^\top)\neq \emptyset$ be given.
Then $f$ is strongly convex over $S$, and taking $c>0$ to be
the modulus of strong convexity, for any $x\in S\cap \Im(A)$,
\[
f(x) - \bar f_A
\leq \frac 1 {2c} \inf_{\psi\in \nf(S)\cap \Ker(A^\top)} \|\nf(x)-\psi\|_1^2.
\]
\end{jlemma}

Before presenting the proof, it can be sketched quite easily.  
Using the Fenchel-Young inequality (cf. \jmlrcref{fact:FY}) and the form of
the dual optimization problem (cf. \jmlrcref{fact:primal_dual}),
primal suboptimality
can be converted into a Bregman divergence in the dual.  If there is strong convexity
in the primal, it allows this Bregman divergence to be converted into a distance
via standard tools in convex optimization (cf. \jmlrcref{fact:sss:2}).  Although
$f$ lacks strong convexity in general, it is strongly convex over any compact set.

\begin{proof}[Proof of \jmlrcref{fact:attain:2ndorder_bound}]
Consider the optimization problem
\[
\inf_{x\in S}\inf_{\substack{\phi\in\R^m\\\|\phi\|_2=1}}
\ip{\nabla^2 f(x)\phi}{\phi}
=
\inf_{x\in S}\inf_{\substack{\phi\in\R^m\\\|\phi\|_2=1}}
\sum_{i=1}^m g''(x_i)\phi_i^2;
\]
since $S$ is compact and $g''$ and $(\cdot)^2$ are continuous, the infimum is attainable.
But $g'' > 0$ and $\phi \neq \bfz_m$, meaning the infimum $c$ is nonzero,
and moreover it is the modulus of strong convexity of $f$ over $S$
\citep[Theorem B.4.3.1.iii]{HULL}.

Now let any $x\in S\cap \Im(A)$ be given,
define $D = \nf(S)\subset \R^m_+$, and for convenience set $K := \Ker(A^\top)$.  
Consider the dual element $\sfP^2_{D\cap K}(\nf(x))$ (which exists since
$D\cap K \neq \emptyset$);
due to the projection, it is dual feasible, and thus it must follow 
from \jmlrcref{fact:primal_dual} that
\[
\bar f_A = \sup\{-f^*(\psi): \psi \in \Phi_A\}
\geq -f^*\left(\sfP^2_{D\cap K}(\nf(x))\right).
\]
Furthermore, since $x\in \Im(A)$, 
\[
\ip{x}{\sfP^2_{D\cap K}(\nf(x))}
= 0.
\]
Combined with the Fenchel-Young inequality (cf. \jmlrcref{fact:FY})
and $x = \nf^*(\nf(x))$,
\begin{align}
f(x) - \bar f_A
&\leq f(x)  + f^*\left(\sfP^2_{D\cap K}(\nf(x))\right)
\notag\\
&= f^*\left(\sfP^2_{D\cap K}(\nf(x))\right) + \ip{\nf(x)}{x} - f^*(\nf(x)) 
\notag\\
&= f^*\left(\sfP^2_{D\cap K}(\nf(x))\right) \!-\! f^*(\nf(x)) 
- \ip{\nf^*(\nf(x))}{\sfP^2_{D\cap K}(\nf(x)) \!-\! \nf(x)}
\label{eq:rate:attain:bregman:1}\\
&\leq \frac 1 {2c} \|\nf(x) - \sfP^2_{D\cap K}(\nf(x))\|_2^2,
\label{eq:rate:attain:bregman:2}
\end{align}
where the last step follows by an application of \jmlrcref{fact:sss:2},
noting that 
both 
$\nf(x)$ and $\sfP^2_{D\cap K}(\nf(x))$ are in $\nf(S)=D$,
and $f$ is strongly convex with modulus $c$ over $S$.
To finish, rewrite $\sfP$ as an infimum and use $\|\cdot\|_2 \leq \|\cdot\|_1$.
\end{proof}

The desired result now follows readily.

\begin{proof}[Proof of \jmlrcref{fact:rate:stiemke}]
Invoking \jmlrcref{fact:attain:hypercube} with $d= f(A\lambda_0)$ immediately
provides a compact tightest axis-aligned rectangle $\cC$ containing the 
initial level set $S := \{x \in \R^m : (f+\iota_{\Im(A)})(x) \leq f(A\lambda_0)\}$.
Crucially, since the objective values never increase,
$S$ and $\cC$ contain every iterate $\{A\lambda_t\}_{t=1}^\infty$.

Applying \jmlrcref{fact:attain:2ndorder_bound} to the set $\cC$ (by 
\jmlrcref{fact:attain:hypercube}, $\nf(\cC)\cap \Ker(A^\top)
\neq \emptyset$),
then for any $t$,
\[
f(A\lambda_t) - \bar f_A 
\leq 
\frac {1}{2c}
\|\nf(A\lambda_t) - \sfP^1_{\nf(\cC) \cap \Ker(A^\top)}(\nf(A\lambda_t))\|_1^2,
\]
where $c >0$ is the modulus of strong convexity of $f$ over $\cC$.

Finally, if there are suboptimal iterates, then $\nf(\cC)\supseteq \nf(S)$ contains
points that are not dual feasible, meaning $\nf(\cC)\setminus \Ker(A^\top) \neq
\emptyset$; since \jmlrcref{fact:attain:hypercube} also provided 
$\nf(\cC)\cap\Phi_A \neq \emptyset$ and $\nf(\cC)$ is a hypercube, it follows 
by \jmlrcref{fact:gamma_p:sanity_check} that
$\gamma(A,\nf(\cC)) > 0$. 
Plugging this into \jmlrcref{fact:rate_ub} 
and using $f(A\lambda_t) \leq f(A\lambda_0)$ gives
\begin{align*}
f(A\lambda_{t+1}) - \bar f_A
& \leq f(A\lambda_t) - \bar f_A -
\frac {\gamma(A,\nf(\cC))^2 \dist^1_{\nf(\cC)\cap\Ker(A^\top)}(\nf(A\lambda_t))^2} {6\eta f(A\lambda_t)}\\
& \leq (f(A\lambda_t) - \bar f_A)\left(1 -
\frac {c\gamma(A,\nf(\cC))^2} {3\eta f(A\lambda_0)}
\right),
\end{align*}
and the result again follows by recursively applying this inequality.
\end{proof}

\begin{jremark}
The key conditions on $g\in \bG$, namely the existence of constants granting
$g \leq \beta g'$ and $g'' \leq \eta g$ within the initial level set,
are much more than are needed in this setting.  Inspecting the presented proofs, it
entirely suffices that on any compact set in $\R^m$, $f$ has quadratic upper and lower 
bounds (equivalently, bounds on the smallest and largest eigenvalues of the Hessian),
which are precisely the weaker conditions used in previous treatments
\citep{bickel_ritov_zakai,convergence_attainability}.

These quantities are therefore necessary for controlling convergence under weak learnability.
To see how the proofs of this section break down in that setting, consider the central Bregman divergence
expression in \jmlreqref{eq:rate:attain:bregman:1}.  What is really granted by attainability is
that every iterate lies well within the interior of $\dom(f^*)$, and therefore these
Bregman divergences, which depend on $\nf^*$, can not become too wild.  On the other hand,
with weak learnability, all dual weights go to zero (cf. \jmlrcref{fact:gordan:derived}), which
means that $\nabla g^* \uparrow \infty$, and thus the upper bound in 
\jmlreqref{eq:rate:attain:bregman:2} ceases to be valid.
As such, another mechanism is required to control this 
scenario, which is precisely the role of $g \leq \beta g'$ and $g'' \leq \eta g$.
\end{jremark}

\subsection{General Setting}
\label{sec:rate:general}
The key development of \jmlrcref{sec:hard_core:general} was that general instances may
be decomposed uniquely into two smaller pieces, one satisfying attainability and the
other satisfying weak learnability, and that these smaller problems behave somewhat 
independently.
This independence is leveraged here to produce convergence rates
relying upon the existing rate analysis for the attainable and weak learnable
cases.
The mechanism of the proof is
as straightforward as one could hope for: decompose the dual distance into the two
pieces, handle them separately using preceding results, and then stitch them back
together.

\begin{jtheorem}
\label{fact:rate:interp}
Suppose $g\in\bG$ and $1\leq |H(A)| \leq m-1$.  Recall from \jmlrcref{sec:hard_core:general}
the partition of the rows of $A$ into $A_0 \in \R^{m_0\times n}$ and $A_+\in \R^{m_+\times n}$, and 
suppose the axes of $\R^m$ are ordered so that
$A=\left[\begin{smallmatrix}A_0 \\ A_+\end{smallmatrix}\right]$.
Set
$\cC_+$ to be the tightest 
axis-aligned rectangle 
$\cC_+\supseteq \{x\in \R^{m_+} : (f+\iota_{\Im(A_+)})(x) \leq f(A\lambda_0)\}$,
and
$w := \sup_{t} 
\|\nf(A_+\lambda_t) - \proj^1_{\nf(\cC_+)\cap\Ker(A^\top_+)}(\nf(A_+\lambda_t))\|_1$.
Then $\cC_+$ is compact, $w<\infty$, $f$ has modulus of strong convexity $c>0$ over
$\cC_+$, and $\gamma(A,\R^{m_0}\times \nf(\cC_+)) > 0$.  Using these terms, for all $t$,
\[
f(A\lambda_t) - \bar f_A \leq 
\frac {2f(A\lambda_0)}{
 (t+1) \min\left\{1,\gamma(A,\R^{m_0}_+ \times \nf(\cC_+))^2/(3\eta(\beta + w/(2c))^2
 )\right\}}.
\]
\end{jtheorem}

The new term, $w$, appears when stitching together the two subproblems.  For choices of
$g\in\bG$ where $\dom(g^*)$ is a compact set, this value is easy to bound; for instance,
the logistic loss, where $\dom(g^*) = [0,1]$, has
$w \leq \sup_{\phi\in\dom(f^*)} \|\phi - \bfz_m\|_1 = m$ (since $\bfz_m\in\dom(f^*)$).
And with the exponential loss,  taking $S := \{\lambda \in \R^n  : f(A\lambda) \leq f(A\lambda_0)\}$ to denote
the initial level set, since $\bfz_m$ is always dual feasible, 
\[
w 
\leq \sup_{\lambda \in S} \|\nf(A\lambda)\|_1 
= \sup_{\lambda \in S} f(A\lambda)
= f(A\lambda_0) = m.
\]

Note that rearranging the rate from \jmlrcref{fact:rate:interp} will provide that
$\cO(1/\epsilon)$ iterations suffice to reach suboptimality $\epsilon$, whereas
the earlier scenarios needed only $\cO(\ln(1/\epsilon))$ iterations.  The exact 
location of the degradation will be pinpointed after the proof, and is related to the 
introduction of $w$.



\begin{proof}[Proof of \jmlrcref{fact:rate:interp}]
By \jmlrcref{fact:megagordan}, $\bar f_{A_+} = \bar f_A$, and the form of $f$
gives $f(A\lambda_t) = f(A_0\lambda_t) + f(A_+\lambda_t)$,
thus
\begin{align}
f(A\lambda_t) - \bar f_A
&= f(A_0\lambda_t) + f(A_+\lambda_t) - \bar f_{A_+}.
\label{eq:rate:interp:blah1}
\end{align}
For the left term, since $g(x) \leq \beta |g'(x)|$,
\begin{equation}
f(A_0\lambda_t) \leq \beta\|\nf(A_0\lambda_t)\|_1
= \beta \|\nf(A_0\lambda_t) - \sfP^1_{\Phi_{A_0}}(\nf(A_0\lambda_t))\|_1,
\label{eq:rate:general:helperdog3}
\end{equation}
which used the fact (from \jmlrcref{fact:megagordan}) that $\Phi_{A_0} = \{\bfz_{m_0}\}$.

For the right term of \eqref{eq:rate:interp:blah1}, recall from
\jmlrcref{fact:megagordan} that $f+ \iota_{\Im(A_+)}$ is
0-coercive,
thus the
level set $S_+ := \{x\in \R^{m_+} : (f+\iota_{\Im(A_+)})(x) \leq f(A\lambda_0)\}$
is compact.  For all $t$, since $f\geq 0$ and the objective values never increase,
\[
f(A\lambda_0) \geq f(A\lambda_t) = f(A_0\lambda_t) + f(A_+\lambda_t)
\geq f(A_+\lambda_t);
\]
in particular, $A_+\lambda_t \in S_+$.
It is crucial that the level set compares against
$f(A\lambda_0)$ and not $f(A_+\lambda_0)$.

Continuing,
\jmlrcref{fact:attain:hypercube} may be applied to $A_+$ with value $d=f(A\lambda_0)$,
which grants a tightest axis-aligned rectangle $\cC_+\subseteq \R^{m_+}$
containing $S_+$, and moreover $\nf(\cC_+)\cap\Ker(A^\top_+)\neq \emptyset$.
Applying \jmlrcref{fact:attain:2ndorder_bound} to $A_+$ and $\cC_+$,
$f$ is strongly convex with modulus $c>0$ over $\cC_+$,
and for any $t$,
\begin{equation}
f(A_+\lambda_t) - \bar f_{A_+}
\leq
\frac 1 {2c} \|\nf(A_+\lambda_t) 
- \proj^1_{\nf(\cC_+)\cap \Ker(A_+^\top)}(\nf(A_+\lambda_t))\|_1^2.
\label{eq:rate:general:helperdog4}
\end{equation}
Next, set 
$w := \sup_{t} \|\nf(A_+\lambda_t) - \proj^1_{\nf(\cC_+)\cap \Ker(A^\top)}(\nf(A_+\lambda_t))\|_1$;
$w<\infty$ since $S_+$ is compact and $\nf(\cC_+)\cap\Ker(A^\top)$ is nonempty.
By the definition of $w$, 
\begin{align*}
\dist^1_{\nf(\cC_+)\cap \Ker(A^\top_+)}(\nf(A_+\lambda_t))^2
\leq w\dist^1_{\nf(\cC_+)\cap \Ker(A^\top_+)}(\nf(A_+\lambda_t)),
\end{align*}
which combined with 
\eqref{eq:rate:general:helperdog4} yields
\begin{equation}
f(A_+\lambda_t) - \bar f_{A_+}
\leq
\frac w {2c}
\dist^1_{\nf(\cC_+)\cap \Ker(A_+^\top)}(\nf(A_+\lambda_t)).
\label{eq:rate:general:helperdog5}
\end{equation}

To merge the subproblem dual distance upper bounds 
\eqref{eq:rate:general:helperdog3}
and 
\eqref{eq:rate:general:helperdog5}
via \jmlrcref{fact:interp:norm_split}, it must be shown that 
$(\R^{m_0}_+\times \nf(\cC_+))\cap \Phi_A \neq \emptyset$.  But this follows by
construction and \jmlrcref{fact:megagordan}, since
$\{\bfz_m\} = \Phi_{A_0} \subseteq\R^m_+$,
$\nf(\cC_+)\cap \Phi_{A_+}\neq \emptyset$
by
\jmlrcref{fact:attain:hypercube}, and the decomposition
$\Phi_A = \Phi_{A_0}\times \Phi_{A_+}$.
Returning to the total suboptimality expression \eqref{eq:rate:interp:blah1},
these dual distance bounds yield
\begin{align*}
f(A\lambda_t) - \bar f_A
&\leq \beta \dist^1_{\Phi_{A_0}}(\nf(A_0\lambda_t))
+ w/(2c) \dist^1_{\nf(\cC_+) \cap \Ker(A_+^\top)}(\nf(A_+\lambda_t))
\\
&\leq (\beta  + w/(2c))\dist^1_{(\R^{m_0}_+\times \nf(\cC_+))\cap \Ker(A^\top)}
(\nf(A\lambda_t)),
\end{align*}
the second step using \jmlrcref{fact:interp:norm_split}.

To finish, note $\R^{m_0}_+\times \nf(\cC_+)$ is polyhedral,
and 
\[
(\R^{m_0}_+\times \nf(\cC_+))\setminus \Ker(A^\top)
\quad \supseteq \quad
\{\nf(A\lambda_t)\}_{t=1}^\infty  \setminus \Ker(A^\top)
\quad \neq \quad \emptyset
\]
since no primal iterate is optimal and thus $\nf(A\lambda_t)$ is not dual feasible
by optimality conditions; combined with the above derivation
$(\R^{m_0}_+\times \nf(\cC_+))\cap \Phi_A \neq \emptyset$,
\jmlrcref{fact:gamma_p:sanity_check} may be applied, meaning 
$\gamma(A,\R^{m_0}_+ \times \nf(\cC_+)) > 0$.  
As such, all conditions of \jmlrcref{fact:rate_ub} are met,
and making use of $f(A\lambda_t) \leq f(A\lambda_0)$,
\begin{align*}
f(A\lambda_{t+1}) - \bar f_A
&\leq f(A\lambda_t) - \bar f_A -
\frac {\gamma(A, \R^{m_0}_+\times \nf(\cC_+))^2 \dist^1_{(\R^{m_0}_+\times \nf(\cC_+))\cap \Ker(A^\top)}(\nf(A\lambda_t))^2} {6\eta  f(A\lambda_t)}\\
&\leq f(A\lambda_t) - \bar f_A -
\frac {\gamma(A, \R^{m_0}_+\times \nf(\cC_+))^2 (f(A\lambda_t) - \bar f_A)^2}
{6  \eta  f(A\lambda_0)  (\beta + w/(2c))^2}.
\end{align*}
Applying \jmlrcref{fact:sss:1} with
\[
\epsilon_t := \frac{f(A\lambda_t) - \bar f_A}{f(A\lambda_0)}
\qquad
\textup{and}
\qquad
r := 
\frac 1 2 
\min\left\{1,
\frac {\gamma(A,\R^{m_0}_+\times \nf(\cC_+))^2}
{3\eta(\beta+w/(2c))^2}\right\}
\]
gives the result.
\end{proof}

In order to produce a rate $\cO(\ln(1/\epsilon))$ under attainability, strong convexity
related the suboptimality to a \emph{squared} dual distance $\|\cdot\|_1^2$ (cf.
\eqref{eq:rate:attain:bregman:2}).  On the other hand, the rate $\cO(\ln(1/\epsilon))$
under weak learnability came from a fortuitous cancellation with the denominator
$f(A\lambda_t)$ (cf. \eqref{eq:rate:gordan_ada:key}),
which is equal to the total suboptimality since
\jmlrcref{fact:gordan:derived} provides
$\bar f_A = 0$.
But in order to merge the subproblem dual distances
via \jmlrcref{fact:interp:norm_split}, the differing properties granting fast rates
must be ignored.  (In the case of attainability, this process introduces $w$.)

This incompatibility is not merely an artifact of the analysis. 
Intuitively, the finite
and infinite margins sought by the two pieces $A_0,A_+$ are in conflict.  For a beautifully
simple, concrete case of this, consider the following matrix, due to
\citet{schapire_adaboost_convergence_rate}:
\[
S :=
\begin{bmatrix}
-1 & +1 \\
+1 & -1 \\
-1 & -1
\end{bmatrix}.
\]
The optimal solution here is to push both coordinates of $\lambda$ unboundedly positive,
with margins approaching $(0,0,\infty)$.  But pushing any coordinate $(\lambda)_i$ too
quickly will increase the objective value, rather than decreasing it.  In fact, this
instance will provide a lower bound, and the mechanism of the proof shows that the 
primal weights
grow extremely slowly, as $\cO(\ln(t))$.
\begin{jtheorem}
\label{fact:rate:lb}
Fix $g=\ln(1+\exp(\cdot))\in\bG$, the logistic loss, and suppose the line search is exact.
Then for any $t\geq 1$, $f(S\lambda_t) - \bar f_S \geq 1/(8t)$.
\end{jtheorem}
(The proof, in \jmlrcref{sec:proof:fact:rate:lb}, is by brute force.)

Finally, note that this third setting does not always entail slow convergence.  Again
taking the view of the rows of $S$ being points $\{-s_i\}_1^3$, consider the effect of
rotating the entire instance around the origin by $\pi/4$.  The optimization 
scenario
is unchanged, however coordinate descent can now be arbitrarily close to the optimum in one
iteration by pushing a single primal weight extremely high.

\subsection*{Acknowledgement}
The author thanks his advisor, Sanjoy Dasgupta, for valuable discussions and support
throughout this project; Daniel Hsu for many discussions, and for introducing
the author to the problem; Indraneel Mukherjee and Robert Schapire for
discussions, and for sharing their related work; Robert Schapire for sharing
early drafts of his book with Yoav Freund;
the JMLR reviewers and editor for their extremely careful reading and comments.
 This work was supported 
by the NSF under grants IIS-0713540 and IIS-0812598.




\ifjmlr
\else
\addcontentsline{toc}{section}{References}
\bibliography{ab}
\fi

\clearpage
\appendix

\section{Common Notation}
\label{sec:notation}

\begin{center}
\begin{tabular}{cp{5in}}
Symbol & Comment\\
\hline
\hline
$\R^m$ & $m$-dimensional vector space over the reals.\\
$\R_+^m$ & Non-negative $m$-dimensional real vectors.\\
$\interior(S)$ & The interior of set $S$.\\
$\R_{++}^m$ & Positive $m$-dimensional real vectors, i.e. $\interior(\R_+^m)$.\\
$\R_{-}^m,\R_{--}^m$ & Respectively $-\R^m_+,-\R^m_{++}$.\\
$\bfz_m,\bfo_m$ & $m$-dimensional vectors of all zeros and all ones, respectively.\\
$\bfe_i$ & Indicator vector: 1 at coordinate $i$, 0 elsewhere.  Context will
provide the ambient dimension.\\
$\Im(A)$& Image of linear operator $A$. \\
$\Ker(A)$ & Kernel of linear operator $A$.\\
\hline
$\iota_S$ & Indicator function on a set $S$:
\[
\iota_S(x) := \left\{
\begin{matrix}
0 & x\in S,\\
\infty & x\not\in S.
\end{matrix}
\right.
\]\\
$\dom(h)$
&Domain of convex function $h$, i.e. the set $\{x\in \R^m: h(x) < \infty\}$.\\
$h^*$
&The Fenchel conjugate of $h$:
\[
h^*(\phi) = \sup_{x\in\dom(h)} \ip{\phi}{x} - h(x).
\]
(Cf. \jmlrcref{sec:dual} and \jmlrcref{sec:shoulders:fenchel}.)\\
$\textrm{0-coercive}$ 
&A convex function 
with all level sets compact is called 0-coercive
(cf. \jmlrcref{sec:hard_core:attainable}).\\
\hline
$\bG_0$ & Basic loss family under consideration (cf. \jmlrcref{sec:setup}).\\
$\bG$ & Refined loss family for which convergence rates are established (cf. \jmlrcref{sec:rate}).\\
$\eta,\beta$ &Parameters corresponding to some $g\in \bG$ (cf. \jmlrcref{sec:rate}).\\
$\Phi_A$ & The general dual feasibility set: $\Phi_A := \Ker(A^\top) \cap \R^m_+$ (cf. \jmlrcref{sec:dual}).\\
$\gamma(A,S)$  
&Generalization of classical weak learning rate (cf. \jmlrcref{sec:wl}).\\
$\bar f_A$ & The minimal objective value of $f\circ A$:
$\bar f_A := \inf_{\lambda} f(A\lambda)$ (cf. \jmlrcref{sec:setup}).\\
$\psi_A^f$ & Dual optimum 
(cf. \jmlrcref{sec:dual}).\\
$\proj^p_S$
& $l^p$ projection onto closed nonempty convex set $S$, with ties broken
in some consistent manner 
(cf. \jmlrcref{sec:wl}).\\
$\dist^p_S$
& $l^p$ distance to closed nonempty convex set $S$: $\dist^p_S(\phi) := \|\phi - \proj^p_S(\phi)\|_p$.
\end{tabular}
\end{center}

\section{Supporting Results from Convex Analysis, Optimization, and Linear Programming}
\label{sec:shoulders}
\subsection{Theorems of the Alternative}
Theorems of the alternative consider the interplay between a matrix (or a few matrices)
and its transpose; they are typically stated as two alternative scenarios, exactly one
of which must hold.  These results usually appear in connection with linear programming,
where Farkas's lemma is used to certify (or not) the existence of solutions.  In the
present manuscript, they are used to establish the relationship between $\Im(A)$ and
$\Ker(A^\top)$, appearing as the first and fourth clauses of the various characterization
theorems in \jmlrcref{sec:hard_core}.

The first such theorem, used in the setting of weak learnability, is perhaps the
oldest theorem of 
\ifjmlr
alternatives~\citep[Bibliographic Notes, Section 5 of Chapter 2]{dantzig_thapa_lp}.
\else
alternatives~\citep[see][bibliographic notes, Section 5 of Chapter 2]{dantzig_thapa_lp}.
\fi
Interestingly,
a streamlined presentation,
using a related optimization problem (which can nearly be written as $f\circ A$ from
this manuscript), can be found in 
\ifjmlr
\citet[Theorem 2.2.6]{borwein_lewis}.
\else
\citet[Theorem 2.2.6]{borwein_lewis}.
\fi

\ifjmlr
\begin{jtheorem}[{Gordan \citep[Theorem 2.2.1]{borwein_lewis}}]
\else
\begin{jtheorem}[{Gordan \citep[Theorem 2.2.1]{borwein_lewis}}]
\fi
\label{fact:gordan}
For any $A\in \R^{m\times n}$, exactly one of the following situations holds:
\begin{align}
&&\exists \lambda \in \R^n&\centerdot A\lambda \in \R^m_{--};\\
&&\exists \phi\in \R^m_+\setminus\{\bfz_m\}&\centerdot A^\top \phi = \bfz_n.
\end{align}
\end{jtheorem}
A geometric interpretation is as follows.  Take the rows of $A$ to be $m$ points in
$\R^n$.  Then there are two possibilities: either there exists an open homogeneous
halfspace containing all points, or their convex hull contains the origin.

Next is Stiemke's Theorem of the Alternative, used in connection with attainability.

\ifjmlr
\begin{jtheorem}[{Stiemke \citep[Exercise 2.2.8]{borwein_lewis}}]
\else
\begin{jtheorem}[{Stiemke \citep[Exercise 2.2.8]{borwein_lewis}}]
\fi
\label{fact:stiemke}
\ifjmlr
For any $A\in\linebreak[4] \R^{m\times n}$,
\else
For any $A\in \R^{m\times n}$,
\fi
exactly one of the following situations holds:
\begin{align}
&&\exists\lambda \in \R^n&\centerdot A\lambda \in \R^m_{-}\setminus\{\bfz_m\};\\
&&\exists \phi\in \R^m_{++}&\centerdot A^\top \phi = \bfz_n.
\end{align}
\end{jtheorem}
The geometric interpretation here is that either there exists a closed homogeneous
halfspace containing all $m$ points, with at least one point interior to the halfspace,
or the relative interior of the convex hull of the points contains the origin
(for the connection to relative interiors, see for instance
\ifjmlr
\citet[Remark A.2.1.4]{HULL}).
\else
\citet[Remark A.2.1.4]{HULL}).
\fi

Finally, a version of Motzkin's Transposition Theorem, which can encode the theorems of
alternatives due to Farkas, Stiemke, and 
Gordan~\citep{benisrael_motzkin}.

\ifjmlr
\begin{jtheorem}[{Motzkin \citep[Theorem 2.16]{dantzig_thapa_lp}}]
\else
\begin{jtheorem}[{Motzkin \citep[Theorem 2.16]{dantzig_thapa_lp}}]
\fi
\label{fact:motzkin}
For any $B\in \R^{z\times n}$ and $C\in \R^{p\times n}$, exactly one of the
following situations holds:
\begin{align}
&&\exists \lambda\in \R^n
&\centerdot B\lambda \in \R^z_{--} \land C\lambda \in \R^p_-,\\
&&\exists \phi_B\in \R^z_+\setminus \{\bfz_z\},\phi_C\in \R^p_+
&\centerdot B^\top \phi_B + C^\top\phi_C  = \bfz_n.
\end{align}
\end{jtheorem}
For this geometric interpretation, take any matrix $A\in\R^{m\times n}$, broken
into two submatrices $B\in \R^{z\times n}$ and $C\in \R^{p\times n}$, with $z+p = m$;
again, consider the rows of $A$ as $m$ points in $\R^n$.  The first possibility
is that there exists a closed homogeneous 
halfspace containing all $m$ points, the $z$ points
corresponding to $B$ being interior to the halfspace.  Otherwise, the origin can be
written as a convex combination of these $m$ points, with positive weight on at least
one element of $B$.

\subsection{Fenchel Conjugacy}
\label{sec:shoulders:fenchel}
The Fenchel conjugate of a function $h$, defined in \jmlrcref{sec:dual}, is
\[
h^*(\phi) = \sup_{x \in \dom(h)} \ip{x}{\phi} - h(x),
\]
where $\dom(h) = \{x : h(x) < \infty\}$.  The main property of the conjugate,
indeed what motivated its definition, is that $\nabla h^*(\nabla h(x)) = x$
\ifjmlr
\citep[Corollary E.1.4.4]{HULL}.
\else
\citep[Corollary E.1.4.4]{HULL}.
\fi
To 
demystify this, differentiate and set to zero the contents of the above $\sup$: the
Fenchel conjugate acts as an inverse gradient map.
For a beautiful description of Fenchel conjugacy, please see 
\ifjmlr
\citet[Section E.1.2]{HULL}.
\else
\citet[Section E.1.2]{HULL}.
\fi

Another crucial property of Fenchel conjugates is 
the Fenchel-Young inequality, simplified here for differentiability 
(the ``if'' can be strengthened to ``iff'' via subgradients).

\ifjmlr
\begin{jproposition}[{Fenchel-Young \cite[Proposition 3.3.4]{borwein_lewis}}]
\else
\begin{jproposition}[{Fenchel-Young \cite[proposition 3.3.4]{borwein_lewis}}]
\fi
\label{fact:FY}
For any convex function $h$ and $x\in\dom(h)$, $\phi\in \dom(h^*)$,
\[
h(x) + h^*(\phi) \geq \ip{x}{\phi},
\]
with equality if $\phi = \nabla h(x)$.
\end{jproposition}

\subsection{Convex Optimization}
Two standard results from convex optimization will help produce convergence
rates; note that these results can be found in many sources.

First, a lemma to convert single-step convergence results into general convergence results.

\begin{jlemma}[Lemma 20 from \citet{shai_singer_weaklearn_linsep}]
\label{fact:sss:1}
Let $1\geq \epsilon_1\geq \epsilon_2\geq \ldots$ be given with 
$\epsilon_{t+1} \leq \epsilon_t - r\epsilon_t^2$ for some $r\in (0,1/2]$.
Then $\epsilon_t \leq (r(t+1))^{-1}$.
\end{jlemma}

Although strong convexity in the primal grants the existence of a lower bounding
quadratic, it grants upper bounds in the dual.  The following result is also standard
in convex analysis, 
see for instance
\ifjmlr
\citet[proof of Theorem E.4.2.2]{HULL}.
\else
\citet[proof of Theorem E.4.2.2]{HULL}.
\fi

\begin{jlemma}[Lemma 18 from \citet{shai_singer_weaklearn_linsep}]
\label{fact:sss:2}
Let $h$ be strongly convex over compact convex set $S$
with modulus $c$. 
Then for any $\phi_1,\phi_1+\phi_2\in\nabla h(S)$,
\[
h^*(\phi_1+\phi_2) - h^*(\phi_1) \leq \ip{\nabla h^*(\phi_1)}{\phi_2} + \frac 1 {2c}
\|\phi_2\|^2_2.
\]
\end{jlemma}

\section{Basic Properties of $g\in\bG_0$}
\label{sec:gfprop}

\begin{jlemma}
\label{fact:gprop}
Let any $g\in\bG_0$ be given.  Then $g$ is strictly convex, $g>0$, $g$ strictly
increases ($g' > 0$), and $g'$ strictly increases.  Lastly,
$\lim_{x\to\infty}g(x) = \infty$.
\end{jlemma}
\begin{proof}
(Strict convexity and $g'$ strictly increases.)
For any $x < y$,
\[
g'(y) = g'(x) + \int_x^y g''(t)dt \geq g'(x) + (y-x) \inf_{t\in[x,y]} g''(t) > g'(x),
\]
thus $g'$ strictly increases, granting strict 
\ifjmlr
convexity~\citep[Theorem B.4.1.4]{HULL}.
\else
convexity~\citep[Theorem B.4.1.4]{HULL}.
\fi

($g$ strictly increases, i.e. $g' > 0$.)
Suppose there exists $y$ with $g'(y) \leq 0$, and choose any $x< y$.
Since $g'$ strictly increases, $g'(x) < 0$.  But that means
\[
\lim_{z\to-\infty} g(z) 
\geq \lim_{z\to-\infty} g(x) + (z-x)g'(x)
= \infty,
\]
a contradiction.


($g>0$.)  If there existed $y$ with $g(y)\leq 0$, then the strict increasing property
would invalidate $\lim_{x\to-\infty}g(x) = 0$.  

($\lim_{x\to\infty}g(x) = \infty$.)  Let any sequence $\{c_i\}_1^\infty \uparrow
\infty$ be given; the result follows by convexity and $g' > 0$, since 
\begin{align*}
\lim_{i\to\infty} g(c_i)
&\geq \lim_{i\to\infty} g(c_1) + g'(c_1)(c_i - c_1)
= \infty.
\qedhere
\end{align*}
\end{proof}

Next, a deferred proof regarding properties of $g^*$ for $g\in\bG_0$.

\begin{proof}[Proof of \jmlrcref{fact:gconj_prop}]
$g^*$ is strictly convex because $g$ is differentiable, and $g^*$ is 
continuously differentiable on $\interior(\dom(g^*))$
because $g$ is strictly 
\ifjmlr
convex~\citep[Theorems E.4.1.1, E.4.1.2]{HULL}.
\else
convex~\citep[Theorems E.4.1.1, E.4.1.2]{HULL}.
\fi

Next, when $\phi <0$: $\lim_{x\to-\infty} g(x) = 0$ grants the existence of $y$
such that for any $x\leq y$, $g(x) \leq 1$, thus
\[
g^*(\phi) = \sup_x \phi x - g(x)\geq \sup_{x\leq y} \phi x - 1 = \infty.
\]
($g > 0$ precludes the possibility of $\infty-\infty$.)

Take $\phi = 0$; then 
\[
g^*(\phi) = \sup_x -g(x) = -\inf_x g(x) = 0.
\]
When $\phi = g'(0)$, by the Fenchel-Young inequality~(\jmlrcref{fact:FY}),
\[
g^*(\phi) = g^*(g'(0)) = 0\cdot g'(0) - g(0) = -g(0).
\]
Moreover $\nabla g^*(g'(0)) = 0$
\ifjmlr
\citep[Corollary E.1.4.3]{HULL},
\else
\citep[Corollary E.1.4.3]{HULL},
\fi
which combined with
strict convexity of $g^*$ means $g'(0)$ minimizes $g^*$.
$g^*$ is 
\ifjmlr
closed~\citep[Theorem E.1.1.2]{HULL},
\else
closed~\citep[Theorem E.1.1.2]{HULL},
\fi
which combined with the above
gives that $\dom(g^*)=[0,\infty)$ or $\dom(g^*) = [0,b]$ for some $b>0$,
and the rest of the form of $g^*$.
\end{proof}

Finally, properties of the empirical risk function $f$ and its conjugate $f^*$.

\begin{jlemma}
\label{fact:fprop}
Let any $g\in \bG_0$ be given.  Then the corresponding 
$f$ is strictly convex, twice continuously differentiable,
and $\nf > \bfz_m$.  Furthermore, $\dom(f^*) = \dom(g^*)^m \subseteq \R^m_+$, 
$f^*(\bfz_m) = 0$, $f^*$ is strictly convex,
$f^*$ is continuously differentiable on the interior of its
domain, and finally $f^*(\phi) = \sum_{i=1}^m g^*(\phi_i)$.
\end{jlemma}
\begin{proof}
First,
\[
f^*(\phi) = \sup_{x\in \R^m}\ip{\phi}{x} - f(x)
= \sup_{x\in \R^m}\sum_{i=1}^m x_i\phi_i - g(x_i)
= \sum_{i=1}^m g^*(\phi_i).
\]
Next, strict convexity of $g^*$ (cf. \jmlrcref{fact:gconj_prop}) means, for
$x\neq y$, $\ip{\nabla g^*(x) - \nabla g^*(y)}{x-y} > 0$ 
\ifjmlr
\citep[Theorem E.4.1.4]{HULL};
\else
\citep[Theorem E.4.1.4]{HULL};
\fi
thus,
given $\phi_1,\phi_2\in \R^m$ with $\phi_1\neq \phi_2$,
strict convexity of $f^*$ follows from
\[
\ip{\nf^*(\phi_1) - \nf^*(\phi_2)}{\phi_1-\phi_2}
= \sum_{i=1}^m \ip{\nabla g^*((\phi_1)_i) - \nabla g^*((\phi_2)_i)}{(\phi_1)_i-(\phi_2)_i}
> 0.
\]
The remaining properties follow from properties of $g$ and $g^*$ (cf.
\jmlrcref{fact:gprop} and \jmlrcref{fact:gconj_prop}).
\end{proof}

\section{Approximate Line Search}
\label{sec:apx_line_search}
This section provides two approximate line search methods for $\Boost$: an iterative
approach, outlined in \jmlrcref{sec:wolfe_conditions} and analyzed in \jmlrcref{sec:wolfe_guarantee},
and a closed form choice, outlined in \jmlrcref{sec:non_iterative_step}.

The iterative approach follows standard line search principles from nonlinear optimization 
\citep{bertsekas_nlp,nocedal_wright}.  It requires no parameters, only the ability to evaluate
objective values and their gradients, and as such is perhaps of greater practical interest.
Due to this, and the fact that its guarantee is just a constant factor worse than the closed form 
method, all convergence analysis will use this choice.

The closed form step size is provided for the sake of comparison to other choices from the boosting
literature.  The drawback, as mentioned above, is the need to know certain parameters,
specifically a second derivative bound, which may be loose.

Before proceeding, note briefly that this section is the only place where boundedness of the entries
of $A$ is used.  Without this assumption, the second derivative upper bounds would contain the
term $\max_{i,j} A_{ij}^2$, which in turn would appear in the various convergence rates of
\jmlrcref{sec:rate}.

\subsection{The Wolfe Conditions}
\label{sec:wolfe_conditions}
Consider any convex differentiable function $h$,
a current iterate $x$, and a descent direction $v$
(that is, $\nabla h(x)^\top v < 0$).    By convexity, the linearization of $h$ at
$x$ in direction $v$, symbolically 
$h(x) + \alpha \nabla h(x) ^\top v$, will lie below the function.  But, by continuity,
it must be the case that, for any $c_1 \in (0,1)$, the ray
$h(x) + \alpha c_1 \nabla h(x)^\top v$, depicted in \jmlrcref{fig:wolfe:pic},
must lie above $h$ for some small region around
$x$; this gives the first Wolfe condition, also known as the Armijo condition
\ifjmlr
(cf. \citet[Equation 3.4]{nocedal_wright} and \citet[Exercise 1.2.16]{bertsekas_nlp}):
\else
(cf. \citet[Equation 3.4]{nocedal_wright} and \citet[Exercise 1.2.16]{bertsekas_nlp}):
\fi
\begin{equation}
h(x + \alpha v)
\leq h(x) + \alpha c_1 \nabla h(x)^\top v.
\label{eq:wolfe:1}
\end{equation}
Unfortunately, this rule may grant only very limited decrease in objective value, since
$\alpha >0 $ can be chosen arbitrarily small and still satisfy the rule; thus, the second
Wolfe condition, also called a curvature condition, which depends on $c_2 \in (c_1,1)$,
forces the step to be farther away:
\begin{equation}
\nabla h(x + \alpha v) ^\top v
\geq c_2 \nabla h(x)^\top v.
\label{eq:wolfe:2}
\end{equation}
This requires the new gradient (in direction $v$) to be closer to 0, mimicking first order
optimality conditions for the exact line search.
Note that the new gradient (in direction $v$) may in fact be 
positive; this does not affect the analysis.

In the case of boosting, with function $f\circ A$, current iterate $\lambda_t$,
direction $v_{t+1} \in \{\pm \bfe_{j_{t+1}}\}$
satisfying $\nabla (f\circ A)(\lambda_t)^\top v_{t+1}
= -\|\nabla (f\circ A)(\lambda_t)\|_\infty$, these conditions become
\begin{align}
(f\circ A)(\lambda_t + \alpha v_{t+1})
&\leq (f\circ A)(\lambda_t) - \alpha c_1 \|\nabla(f\circ A)(\lambda_t)\|_\infty,
\label{eq:wolfe:boosting:1}
\\
\nabla (f\circ A)(\lambda_t + \alpha v_{t+1}) ^\top v_{t+1}
&\geq -c_2 \|\nabla (f\circ A)(\lambda_t)\|_\infty.
\label{eq:wolfe:boosting:2}
\end{align}
An algorithm to find a point satisfying these conditions, presented in
\jmlrcref{fig:wolfe}, is simple enough:
grow $\alpha$ as quickly as possible, and then bisect backwards for 
a satisfactory point.   As compared with the presentation in 
\ifjmlr
\citet[Algorithm 3.5]{nocedal_wright},
\else
\citet[Algorithm 3.5]{nocedal_wright},
\fi
$\alpha_{\max}$ is searched for rather than 
provided, and convexity removes the need for interpolation.

\begin{figure}
\begin{center}
\framebox[0.9\textwidth][c]{
\begin{minipage}{0.85\textwidth}
\textbf{Routine} \Wolfe.\\
\textbf{Input} Convex function $h$, iterate $x$,
descent direction $v$.\\
\textbf{Output} Step size $\alpha > 0$ satisfying
\jmlreqref{eq:wolfe:1} and \jmlreqref{eq:wolfe:2}.

\begin{enumerate}
\item Bracketing step.
\begin{enumerate}
\item Set $\alpha_{\max} := 1$.
\item While $\alpha_{\max}$ satisfies \jmlreqref{eq:wolfe:1}:
\begin{itemize}
\item 
Set $\alpha_{\max} := 2\alpha_{\max}$.
\end{itemize}
\end{enumerate}
\item Bisection step.
\begin{enumerate}
\item Set $\alpha_{\min} := 0$ and $\alpha := \alpha_{\max}/2$.
\item While $\alpha$ does not satisfy \jmlreqref{eq:wolfe:1} and \jmlreqref{eq:wolfe:2}:
\begin{enumerate}
\item If $\alpha$ violates \jmlreqref{eq:wolfe:1}:
\begin{itemize}
\item Set $\alpha_{\max} := \alpha$.
\end{itemize}
\item Else, $\alpha$ must violate \jmlreqref{eq:wolfe:2}:
\begin{itemize}
\item Set $\alpha_{\min} := \alpha$.
\end{itemize}
\item Set $\alpha := (\alpha_{\min} + \alpha_{\max}) / 2$.
\end{enumerate}
\item Return $\alpha$.
\end{enumerate}
\end{enumerate}
\end{minipage}
}
\end{center}
\caption{Bracketing and bisecting search for step size satisfying Wolfe
conditions.}
\label{fig:wolfe}
\end{figure}

\begin{jproposition}
Given a continuously differentiable convex bounded below 
function $h$, iterate $x$, and direction
$v$,
$\Wolfe$ terminates with an $\alpha>0$ satisfying
\jmlreqref{eq:wolfe:1} and \jmlreqref{eq:wolfe:2}.
\end{jproposition}
\begin{proof}
The bracketing search must terminate: $v$ is a descent direction, so the 
linearization at $\lambda_{t-1}$ with slope $c_1 \nabla h(x)^\top v$
will eventually intersect $h$ (since $h$ it is bounded below). 

\begin{figure}
\begin{center}
\includegraphics[width=0.7\textwidth]{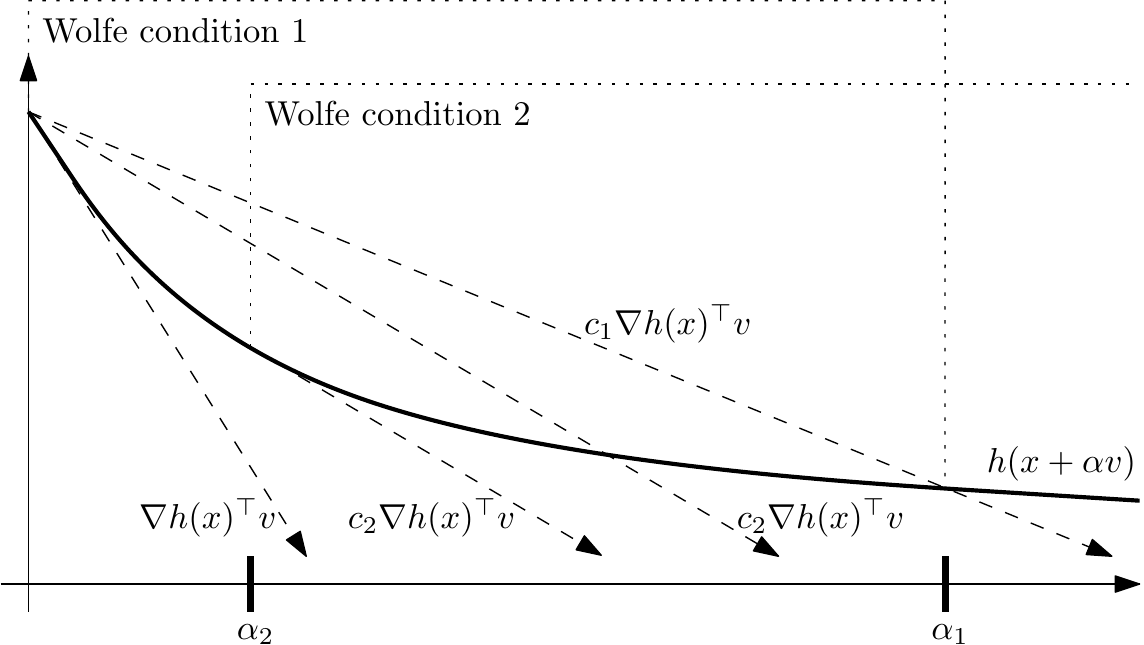}
\end{center}
\caption{The mechanism behind $\Wolfe$: the set of points satisfying
\jmlreqref{eq:wolfe:1} and \jmlreqref{eq:wolfe:2} is a closed interval, and bisection will find interior
points.  In this figure, dashed lines denote various relevant slopes.}
\label{fig:wolfe:pic}
\end{figure}

The remainder of this proof is illustrated in \jmlrcref{fig:wolfe:pic}.
Let $\alpha_1$ be the greatest positive real satisfying \jmlreqref{eq:wolfe:1};
due to convexity, every $\alpha\geq 0$ satisfying this first condition
must also satisfy $\alpha \in [0,\alpha_1]$.  Crucially, $\alpha_1 < \alpha_{\max}$.

Next, let $\alpha_2$ be the smallest positive real satisfying \jmlreqref{eq:wolfe:2};
existence of such a point follows from the existence of points satisfying both
Wolfe conditions 
\ifjmlr
\citep[Lemma 3.1]{nocedal_wright}.
\else
\citep[Lemma 3.1]{nocedal_wright}.
\fi
By convexity,
\[
\ip{\nabla h(x + \alpha v) - \nabla h(x)}{v} \geq 0,
\]
and therefore every $\alpha \geq 0$ satisfying \jmlreqref{eq:wolfe:2} must satisfy
$\alpha \geq \alpha_2$.

Finally, $\alpha_1\neq \alpha_2$, since $c_1< c_2$, meaning
\[
\nabla h(x+\alpha_1 v)^\top v 
= c_2 \nabla h(x)^\top v 
< c_1 \nabla h(x)^\top v
< \nabla h(x+\alpha_2 v)^\top v.
\]

Combining these facts, the interval $[\alpha_2,\alpha_1]$ is precisely the set of
points which satisfy \jmlreqref{eq:wolfe:boosting:1} and \jmlreqref{eq:wolfe:2}.  
The bisection search maintains the invariants 
$\alpha_{\min} \leq \alpha_2$ and $\alpha_{\max} \geq \alpha_1$,
meaning no valid solution is ever thrown out: $[\alpha_2,\alpha_1]\subseteq
[\alpha_{\min},\alpha_{\max}]$.
$[\alpha_2,\alpha_1]$ has nonzero width (since $\alpha_1 \neq \alpha_2$),
and every bisection step halves the width
of $[\alpha_{\min},\alpha_{\max}]$, thus the procedure terminates.
\end{proof}

\subsection{Improvement Guaranteed by $\Wolfe$ Search}
\label{sec:wolfe_guarantee}
The following proof, adapted from 
\ifjmlr
\citet[Lemma 3.1]{nocedal_wright},
\else
\citet[Lemma 3.1]{nocedal_wright},
\fi
provides the
improvement gained by a single line search step.  The usual proof depends on a Lipschitz
parameter on the gradient, which is furnished here by
$g''(x) \leq \eta g(x)$.

\ifjmlr
\begin{jproposition}[{Cf. \citet[Lemma 3.1]{nocedal_wright}}]
\else
\begin{jproposition}[{Cf. \citet[Lemma 3.1]{nocedal_wright}}]
\fi
\label{fact:wolfe_guarantee}
Fix any $g\in\bG$.
If $\alpha_{t+1}$ is chosen by $\Wolfe$ applied to
function $f\circ A$ at iterate $\lambda_t$ in direction $v_{t+1}$ with
$c_1 = 1/3$ and $c_2 = 1/2$, then
\[
f(A(\lambda_t + \alpha_{t+1} v_{t+1}))
\leq 
f(A\lambda_t) - \frac {\|A^\top\nf(A\lambda_t)\|_\infty^2}{6\eta f(A\lambda_t)}.
\]
\end{jproposition}
\begin{proof}
First note that every $\alpha\in[0,\alpha_{t+1}]$ satisfies
\[
f(A(\lambda_{t} + \alpha v_{t+1})) \leq f(A\lambda_{t}).
\]
By the fundamental theorem of calculus,
\begin{align*}
&
(\nabla (f\circ A)(\lambda_t + \alpha_{t+1} v_{t+1}) - \nabla (f\circ A)(\lambda_t))^\top v_{t+1}
\\
&\hspace{1in}=
\int_0^{\alpha_{t+1}} v_{t+1}^\top \nabla ^2 (f\circ A)(\lambda_t + \alpha v_{t+1}) v_{t+1}d\alpha
\\
&\hspace{1in}\leq \alpha_{t+1}\sup_{\alpha \in [0,\alpha_{t+1}]}
\sum_{i=1}^m g''(\bfe_i^\top A (\lambda_{t} + \alpha v_{t+1})) (A_{ij_{t+1}})^2
\\
&\hspace{1in}\leq \eta \alpha_{t+1}\sup_{\alpha \in [0,\alpha_{t+1}]}
\sum_{i=1}^m g(\bfe_i^\top A (\lambda_{t} + \alpha v_{t+1}))
\\
&\hspace{1in}\leq \eta \alpha_{t+1} f(A\lambda_t),
\end{align*}
which used boundedness of the entries in $A$.

The rest of the proof continues as in 
\ifjmlr
\citet[Theorem 3.2]{nocedal_wright}.
\else
\citet[Theorem 3.2]{nocedal_wright}.
\fi
Specifically,
subtracting $\nabla (f\circ A)(\lambda_t)^\top v_{t+1}$ from both sides of
\jmlreqref{eq:wolfe:boosting:2} yields
\[
(\nabla (f\circ A) (\lambda_t + \alpha_{t+1} v_{t+1}) - \nabla (f\circ A)(\lambda_t))^\top v_{t+1}
\geq (c_2 - 1)\nabla (f\circ A)(\lambda_t)^\top v_{t+1}.
\]
Combining these two gives
\[
\alpha_{t+1}
\geq \frac{(c_2-1)\nabla (f\circ A)(\lambda_t)^\top v_{t+1}}{\eta f(A\lambda_t)}
= \frac{(1-c_2)\|\nabla(f\circ A)(\lambda_t)\|_\infty}{\eta f(A\lambda_t)}.
\]
Plugging this into \jmlreqref{eq:wolfe:boosting:1} yields
\[
(f\circ A)(\lambda_t + \alpha_{t+1} v_{t+1})
\leq (f\circ A)(\lambda_t) - 
\frac {c_1(1-c_2) \|\nabla(f\circ A)(\lambda_t)\|_\infty^2} {\eta f(A\lambda_t)}.
\qedhere
\]
\end{proof}

Note briefly that the simpler iterative strategy of backtracking line search 
is doomed to require knowledge of the sorts of parameters appearing in the closed 
form choice.

\subsection{Non-iterative Step Selection}
\label{sec:non_iterative_step}
The same techniques from the proof of \jmlrcref{fact:wolfe_guarantee} can provide a closed
form choice of $\alpha_t$.  In particular, it follows that any 
$\alpha \in \{\alpha \geq 0 : f(A\lambda_t) \geq f(A(\lambda_t+\alpha v_{t+1}))\}$
is upper bounded by the quadratic
\[
f(A(\lambda_t+\alpha v_{t+1}))
\leq 
f(A\lambda_t) - \alpha \|A^\top\nf(A\lambda_t)\|_\infty  + \frac {\alpha^2 \eta f(A\lambda_t)}{2}.
\]
This quadratic is minimized at
\[
\alpha' := \frac {\|A^\top \nf(A\lambda_t)\|_\infty}{\eta f(A\lambda_t)};
\]
moreover, this minimum is attained within the interval above,
which in particular implies
\[
f(A(\lambda_t + \alpha' v_{t+1}))
\leq 
f(A\lambda_t) - \frac {\|A^\top\nf(A\lambda_t)\|_\infty^2}{2\eta f(A\lambda_t)}.
\]
When $\eta$ is simple and tight, this yields a pleasing expression (for instance, $\eta=1$
when $g = \exp(\cdot)$).  In general, however, $\eta$ might be hard to calculate, or simply
very loose, in which case performing a line search like $\Wolfe$ is preferable.

\section{Approximate Coordinate Selection}
\label{sec:apx_coordinate_selection}
Selecting a coordinate $j_t$ translates into selecting some hypothesis $h_t\in\cH$; this 
is in fact a key strength of boosting, since $A$ need not be written down, and a weak
learning oracle can select $h_t\in\cH$.  But for certain hypothesis classes $\cH$, it may be impossible
to guarantee $h_t$ is truly the best choice.

Observe how these statements translate into gradient descent.  Specifically, the choice
$v_{t+1}$ made by boosting satisfies
\[
v_{t+1}^\top \nabla(f\circ A)(\lambda_t) = v_{t+1}^\top A^\top \nf(A\lambda_t) = -\|A^\top \nf(A\lambda_t)\|_\infty.
\]
On the other hand, the usual choice 
$v=-\nabla (f\circ A)(\lambda_t) / \|A^\top \nf(A\lambda_t)\|_2$
of 
gradient descent 
($l^2$ steepest descent)
grants
\[
v^\top \nabla(f\circ A)(\lambda_t) = -\|A^\top \nf(A\lambda_t)\|_2;
\]
note that this choice of $v$ is potentially a dense vector.

\begin{jremark}
Suppose the relaxed condition that the weak learner need merely have any 
correlation over the provided distribution; in optimization terms, the returned
direction $v$ satisfies
\[
v^\top \nabla(f\circ A)(\lambda_t) < 0.
\]
This choice is not sufficient to guarantee convergence, let alone any reasonable
convergence rate.  As an example boosting instance, consider either of the
matrices
\begin{align*}
A_1 
&:= \begin{bmatrix}
-1 & +1 & 0 \\
+1 & -1 & 0 \\
-1 & -1 & 0 \\
0 & 0 & -1
\end{bmatrix},
&
A_2 
&:= \begin{bmatrix}
-1 & +1 & -1 \\
+1 & -1 & -1 \\
-1 & -1 & -1
\end{bmatrix},
\end{align*}
the first of which uses confidence-rated predictors, the second of which is weak learnable;
note that both instances embed the matrix $S$ due to
\citet{schapire_adaboost_convergence_rate}, used for lower bounds in \jmlrcref{sec:rate:general}.

For either instance, $\bfe_1,\bfe_2,\bfe_1,\bfe_2,\bfe_1,\ldots$ is a sequence
of descent directions.  But, for either matrix, to approach optimality, the weight on
the third column must go to infinity.
\end{jremark}

A first candidate fix is to choose some appropriate $c_0>0$, and require
\[
v^\top \nabla (f\circ A)(\lambda_t) \leq -c_0\|\nf(A\lambda_t)\|_1;
\]
but note, by
\jmlrcref{fact:wl:Phi_A_point,fact:gordan:derived}, that this is only possible under 
weak learnability.  (Dropping the term $\|\nf(A\lambda_t)\|_1$ also fails;
suppose $A$ grants a minimizer $\bar \lambda$: plugging this in makes the left hand
side exactly zero, and continuity thus grants arbitrarily small values.)

Instead consider requiring the weak learning oracle to return some hypothesis at least
a fraction $c_0\in(0,1]$ as good as 
the best weak learner in the class; written in the present
framework, the direction $v$ must satisfy
\[
v^\top \nabla(f\circ A)(\lambda_t) \leq -c_0\|A^\top\nf(A\lambda_t)\|_\infty.
\]
Inspecting the proof of \jmlrcref{fact:rate_ub}, it follows that this approximate selection would
simply introduce the constant $c_0^2$ in all rates, but would not degrade their asymptotic relationship
to suboptimality $\epsilon$.



\section{Generalizing the Weak Learning Rate}
\label{sec:wl:addendum}
\subsection{Choosing a Generalization to $\gamma$}
Any generalization $\gamma'$ of $\gamma$ should satisfy the following properties.
\begin{itemize}
\item When weak learnability holds, $\gamma' = \gamma$.
\item For any boosting instance, $\gamma' \in (0,\infty)$.
\item $\gamma'$ provides an expression similar to 
\eqref{eq:gamma_p:gradient:rearrange}, which allows the full gradient to be
converted into a notion of suboptimality in the dual.
\end{itemize}

Taking the form of the classical weak learning rate from \eqref{eq:wl:classical} as a
model, the template generalized weak learning rate is
\[
\gamma'(A,S,C,D) := \inf_{\phi \in S\setminus C} \frac {\|A^\top\phi\|_\infty}
{\inf_{\psi\in S\cap D} \|\phi -\psi\|_1},
\]
for some sets $S$, $C$, and $D$ (for instance, the classical weak learning rate uses
$S=\R^m_+$ and $C=D=\{\bfz_m\}$).  In order to provide an expression
similar to \eqref{eq:gamma_p:gradient:rearrange}, the domain of the infimum must include
every suboptimal dual iterate $\nf(A\lambda_t)$.

Any choice $C$ which does not include all of $\Ker(A^\top)$ is immediately problematic:
this allows $\phi \in S\cap\Ker(A^\top)$ to be selected,
whereby $A^\top\phi = \bfz_m$ and $\gamma' = 0$.
But note that without being careful about $D$, it is still possible to force
the value 0.

\begin{jremark}
Another generalization is to define
\begin{equation}
\gamma''(A) 
:= \gamma'(A,\R^m_+,\Ker(A^\top), \{\psi_A^f\})
=\inf_{\phi\in\R^m_+\setminus\Phi_A} 
\frac{\|A^\top \phi\|_\infty}{\|\phi - \psi_A^f\|_1}.
\label{eq:wl:gen_template}
\end{equation}
This form agrees with the original $\gamma$ when weak learnability holds, and
will lead to a very convenient analog to \eqref{eq:gamma_p:gradient:rearrange}.

Unfortunately, $\gamma''$ may be zero.
Specifically, take the matrix $S$
defined in \jmlrcref{sec:rate:general}, due to \citet{schapire_adaboost_convergence_rate}, where
\[
\psi_S^f = g'(0) \left[
\begin{smallmatrix}1 \\ 1 \\ 0\end{smallmatrix}
\right].
\]
Furthermore, for any $\alpha \in (0,1)$, define
\begin{align*}
\phi_a &:= 
\alpha \left[\begin{smallmatrix}0 \\ 0 \\ 1\end{smallmatrix}\right] \in \Im(S);
&
\psi_\alpha &:=
(1-\alpha) \left[\begin{smallmatrix}1/2 \\ 1/2 \\ 0\end{smallmatrix}\right] + \psi_S^f \in \Ker(S^\top).
\end{align*}
Then 
\[
\inf_{\phi \in \R^m_+\setminus\Ker(S^\top)} \frac{\|S^\top\phi\|_\infty}
{\|\phi-\psi_S^f\|_1}
\leq
\inf_{\alpha \in (0,1)} \frac{\|S^\top(\phi_\alpha + \psi_\alpha)\|_\infty}
{\|\phi_\alpha + \psi_\alpha -\psi_S^f\|_1}
=
\inf_{\alpha \in (0,1)} \frac{\left\|\left[\begin{smallmatrix} -\alpha \\ -\alpha \end{smallmatrix}\right]\right\|_\infty}
{1}
=0.
\qedhere
\]
\end{jremark}


The natural correction to these worries is to set $C = D=\Ker(A^\top)$.
But there is still sensitivity due to $S$.

\begin{jremark}
Set $A := \bfo_2$, meaning $\Ker(A^\top) = \{z(1,-1) : z\in \R\}$,
and $S = B(\bfo_2,\sqrt{2})$, the
ball of radius $\sqrt{2}$ around $\bfo_2$; note that $S\cap \Ker(A^\top) = \bfz_2$.
Consider $\gamma'(A,S,\Ker(A^\top),\Ker(A^\top))$,
and the sequence $\{\phi_i\}_{i=1}^\infty$ where
\[
\phi_i = \bfo_2 - \frac 1 {\sqrt{i^2+1}} \begin{bmatrix}
i + 1 \\ i - 1
\end{bmatrix}.
\]
Note that $\|\phi_i - \bfo_2\|_2 = \sqrt{2}$, thus $\phi_i \in S$.  Furthermore,
$A^\top \phi_i \neq 0$, so $\phi_i \not\in S\cap \Ker(A^\top)$.  As such,
\begin{align}
\gamma'(A,S,\Ker(A^\top),\Ker(A^\top)) 
&\leq \inf_i 
\frac
{\|A^\top\phi_i\|_\infty}
{\|\phi_i - \sfP^1_{S\cap\Ker(A^\top)}(\phi_i)\|_1}
\\
&=
\frac {\|\bfo_2^\top \left(\bfo_2 \sqrt{i^2+1} - \left[\begin{smallmatrix} i+1\\i-1\end{smallmatrix}\right]\right)\|_\infty}
{\|\bfo_2\sqrt{i^2+1} - \left[\begin{smallmatrix} i+1\\i-1\end{smallmatrix}\right]\|_1}.
\label{eq:wl:breakage:helperdog1}
\end{align}
Using $\sqrt{y} \leq (1+y)/2$,
the numerator has upper bound
\begin{align*}
\|\bfo_2^\top \left(\bfo_2\sqrt{i^2+1} 
- \left[\begin{smallmatrix} i+1\\i-1\end{smallmatrix}\right]\right)\|_\infty
&
= |2\sqrt{i^2+1}- 2i|
\\
&
= 2i(\sqrt{1+i^{-2}} - 1)
\\
&
\leq 2i((2+i^{-2})/2 - 1)
= 1/i.
\end{align*}
The denominator is 
\begin{align*}
\|\bfo_2 \sqrt{i^2+1}- \left[\begin{smallmatrix} i+1\\i-1\end{smallmatrix}\right]\|_1
&=
|\sqrt{i^2+1} -(i+1)| + |\sqrt{i^2+1}  - (i-1)|
\\
&=
((i+1)-\sqrt{i^2+1}) + (\sqrt{i^2+1} - (i-1))
\\
&
= 2.
\end{align*}
Thus \eqref{eq:wl:breakage:helperdog1} is bounded above by $\inf_i (2i)^{-1} = 0$.
\end{jremark}

The difficulty here was the curvature of $S$, which allowed elements arbitrarily close
to $\Ker(A^\top)$ without actually being inside this subspace.  
This possibility is averted in this manuscript by requiring polyhedrality of $S$.
This choice is sufficiently rich to allow
the various dual-distance upper
bounds of \jmlrcref{sec:rate}.




\subsection{Proof of \Cref{fact:gamma_p:sanity_check}}
The proof of \jmlrcref{fact:gamma_p:sanity_check} requires a few steps,
but the strategy
is straightforward.  First note that $\gamma(A,S)$ can be rewritten as
\begin{align}
\gamma(A,S)
&= \inf_{\phi\in S\setminus\Ker(A^\top)}
\frac {\|A^\top\phi\|_\infty}{\|\phi - \sfP^1_{S\cap \Ker(A^\top)}(\phi)\|_1}
\notag \\
&= \inf_{\phi\in S\setminus\Ker(A^\top)}
\frac {\|A^\top(\phi-\sfP^1_{S\cap \Ker(A^\top)}(\phi))\|_\infty}
{\|\phi - \sfP^1_{S\cap \Ker(A^\top)}(\phi)\|_1}
\notag \\
&=
\inf\left\{
\frac {\|A^\top v\|_\infty }{\|v\|_1}
: v\in \R^m \setminus\{\bfz_m\},
\exists \phi \in S 
\centerdot 
v = \phi - \sfP^1_{S\cap \Ker(A^\top)}(\phi)
\right\}
\notag\\
&=
\inf\left\{
\|A^\top v\|_\infty
: \|v\|_1 = 1,
\exists \phi \in S ,\exists c>0
\centerdot 
cv = \phi - \sfP^1_{S\cap \Ker(A^\top)}(\phi)
\right\}
,
\label{eq:gamma2:precursor}
\end{align}
where the second equivalence used $A^\top \sfP^1_{S\cap \Ker(A^\top)}(\phi) = \bfz_n$.

In the final form, $v\not\in\Ker(A^\top)$, and so $A^\top v \neq \bfz_n$; that is
to say, the infimand is positive for every element of its domain.  The difficulty
is that the domain of the infimum, written in this way, is not obviously closed; thus
one can not simply assert the infimum is attainable and positive.

The goal then will be to reparameterize the infimum to have a compact domain.
For technical convenience, the result will be mainly proved for the $l^2$ norm (where
projections behave nicely), and
norm equivalence will provide the final result.

\begin{jlemma}
\label{fact:gamma2:power}
Given $A\in\R^{m\times n}$ and a polyhedron $S\subseteq \R^m$
with $S\cap \Ker(A^\top)\neq \emptyset$
and $S\setminus \Ker(A^\top)\neq \emptyset$,
\begin{equation}
\inf\left\{
\frac 
{\|A^\top (\phi - \sfP^2_{S\cap \Ker(A^\top)}(\phi))\|_2}
{\|\phi - \sfP^2_{S\cap \Ker(A^\top)}(\phi)\|_2}
:
\phi \in S\setminus \Ker(A^\top)
\right\}  > 0.
\label{eq:gamma2:power}
\end{equation}
\end{jlemma}
To produce the desired reparameterization of this infimum, the following characterization
of polyhedral sets will be used.

\begin{jdefinition}
For any nonempty polyhedral set $S\subseteq \R^m$,
let $\scrH_S$ index a finite (but possibly empty) collection of affine functions
$g_\alpha : \R^m \to \R$
so that $S = \cap_{\alpha\in\scrH_S} \{x \in \R^m : g_\alpha(x) \leq 0\}$ (with
the convention that $S=\R^m$ when $\scrH_S = \emptyset$).
For any $x\in S$, let $\cI_S(x)$ denote the \emph{active set} for $x$: $\alpha \in
\cI_S(x)$ iff $g_\alpha(x) = 0$.
Lastly, define a relation $\sim_S$ over points in $S$: given $x,y\in S$,
$x\sim_S y$ iff $\cI_S(x) = \cI_S(y)$.
Observe that $\sim_S$ is an equivalence relation over points within $S$,
and let $\cC_S$ be the set of equivalence classes.
\end{jdefinition}

The equivalence relation $\sim_S$ thus partitions $S$ into the members of $\cC_S$,
each of which has a very convenient structure.

\begin{jlemma}
\label{fact:gamma2:decomp_prop}
Let a polyhedral set $S\subseteq \R^m$ be given,
and fix a nonempty $F\in \cC_S$. Then $F$ is convex, and $F$ is equal to its relative 
interior (i.e., $F = \ri(F)$).  Finally, fixing an arbitrary $z_0\in F$,
the normal cone at any point $z\in F$ is orthogonal to the vector space
parallel to the affine hull of $F$ 
(i.e., $N_F(z) =(\aff(F) - \{z\})^\perp = (\aff(F) - \{z_0\})^\perp$).
\end{jlemma}

Throughout the remainder of this section, normal and tangent cones will be considered
at points within a set $F\in \cC_S$.  As \jmlrcref{fact:gamma2:decomp_prop}
establishes, any set $F\in \cC_S$ is \emph{relatively open} ($F = \ri(F)$), however,
the required properties of normal and tangent cones, as developed by
\citet[Sections A.5.2 and A.5.3]{HULL}, suppose \emph{closed} convex sets.
But it is always the case that
$\ri(F) = \ri(\cl(F))$
\citep[Proposition A.2.1.8]{HULL}; as such, the normal and tangent cones at the desired
relative interior points may just as well be constructed against $\cl(F)$, and thus
the aforementioned properties safely hold.

\begin{proof}
If $S=\R^m$ (meaning $\scrH_S$ is empty) or $\dim(F) = 0$ ($F$ is a single point),
everything follows directly, thus suppose $S \neq \R^m$, and fix
a nonempty $F\in \cC_S$ with $\dim(F) > 0$.

Let any $x_0,x_1\in F$ and $\beta \in [0,1]$ be given,
and define $x_\beta := (1-\beta) x_0 + \beta x_1$.
Since each $g_\alpha$ defining $S$ is affine,
\begin{equation}
g_\alpha(x_\beta) = (1-\beta) g_\alpha(x_0) + \beta g_\alpha(x_1).
\label{eq:convex_decomp:affine}
\end{equation}
By construction of $\cC_S$,
$g_\alpha(x_0) = 0$ iff $g_\alpha(x_1)= 0$ and otherwise both are negative,
thus $g_\alpha(x_\beta) = 0$ iff $g_\alpha(x_0) = g_\alpha(x_1)=0$,
meaning $\cI_S(x_\beta) = \cI_S(x_0) = \cI_S(x_1)$, so $x_\beta \in F$ and $F$ is
convex.

Now 
let any $y_0\in F$ be given; $y_0\in\ri(F)$ when there exists a $\delta>0$
so that
\begin{equation}
B(y_0,\delta) \cap \aff(F) \subseteq F
\label{eq:defn:ri}
\end{equation}
\citep[Definition A.2.1.1]{HULL}.
To this end, first define $\delta$ to be half the distance to the closest hyperplane
defining $S$ which is not active for $y_0$:
\[
\delta := \frac 1 2  \min_{\alpha \in \scrH_S\setminus \cI_S(y)} \min\{
\|y' - y_0\|_2 : y' \in \R^m, g_\alpha(y') = 0\}.
\]
Since there are only finitely many such hyperplanes, and the distance to each
is nonzero, $\delta > 0$.   Let any $y_\beta \in B(y,\delta) \cap \aff(F)$ be given;
by definition of $\aff(F)$, there must exist $\beta \in \R$ and $y_1 \in F$ so
that $y_\beta = (1-\beta)y_0 + \beta y_1$.  By \eqref{eq:convex_decomp:affine},
for any $\alpha\in \cI_S(y_0) = \cI_S(y_1)$,
\[
g_\alpha(y_\beta) = (1-\beta) g_\alpha(y_0) + \beta g_\alpha(y_1) = 0.
\]
On the other hand, for any $\alpha \in \scrH_S \setminus \cI_S(y_0)$, it must be the
case that $g_\alpha(y_\beta) < 0$, since $y_\beta \in B(y_0,\delta)$, and due to
the choice
of $\delta$.  Returning to the definition of relative interior in \eqref{eq:defn:ri},
it follows that $y_0 \in \ri(F)$, and $\ri(F) = F$ since $y_0\in F$ was arbitrary.


For the final property, for any $z_0,z\in \ri(F) = F$, the tangent cone
$T_F(z)$ has form $(\aff(F) - \{z\})$ \citep[see][Proposition A.5.2.1 and discussion
within Section A.5.3]{HULL}, and note 
$\aff(F) - \{z\} = \aff(F) + \{z_0 - z\} - \{z_0\} =  \aff(F) - \{z_0\}$.
Lastly, $N_F(z) = T_F(z)^\perp$ \citep[Proposition A.5.2.4]{HULL}.
\end{proof}

The relevance to \eqref{eq:gamma2:power} and \eqref{eq:gamma2:precursor} is that
projections from polyhedron $S$ onto $S\cap\Ker(A^\top)$ (itself a polyhedron,
as is verified in the proof of \jmlrcref{fact:gamma2:power})
must land on some equivalence class of $\cC_{S\cap \Ker(A^\top)}$, and these
projections are easily characterized.
\begin{jlemma}
\label{fact:gamma2:proj_onto_inside}
Let any nonempty polyhedra $S\subseteq \R^m$ and $K\subseteq \R^m$ be given,
and fix any nonempty $F\in\cC_{S\cap K}$ and $x_F \in F$.
Define 
\begin{align*}
P_F &:=
\{
c(\phi - \sfP^2_{S\cap K}(\phi)) : c > 0, \phi \in S, \sfP^2_{S\cap K}(\phi) \in F
\},
\\
D_F &:=
N_F(x_F) \cap \{y - x_F : y\in \R^m, \forall \alpha \in \cI_S(x_F) \centerdot g_\alpha(y) \leq 0\},
\end{align*}
where $N_F(x_F)$ is the normal cone of $F$ at $x_F$.
Then $P_F = D_F$.
\end{jlemma}
Note that the final active set $\cI_S(x_F)$ is with respect to $S$, not $S\cap K$.
\begin{proof}
($\subseteq$) Let any $\phi\in S$ with $\psi := \sfP^2_{S\cap K}(\phi) \in F$ be given,
where the latter is well-defined since $F$ and hence $S\cap K$ are nonempty.
By \jmlrcref{fact:gamma2:decomp_prop}, $\psi\in \ri(F)$, and $N_F(\psi) = N_F(x_F)$,
meaning $\phi - \psi \in N_F(x_F)$ \citep[Proposition A.5.3.3]{HULL}.
Since $\phi \in S$, 
for any $\alpha \in \cI_S(\psi) = \cI_S(x_F)\subseteq \scrH_S$,
$g_\alpha(\phi)\leq 0$, so
\begin{align*}
\phi - \psi 
\in \{y\in \R^m : g_\alpha(y) \leq 0\} - \{\psi\}
&= \big(\{y\in \R^m : g_\alpha(y) \leq 0\} -\{ \psi - x_F\}\big) - \{x_F\}
\\
&= \{y\in \R^m : g_\alpha(y) \leq 0\} - \{x_F\},
\end{align*}
the final equality following since $g_\alpha(x_F) = g_\alpha(\psi) = 0$ and
$g_\alpha$ defines an affine hyperplane, meaning the corresponding 
affine halfspace is closed 
under translations by $\psi - x_F$.
This holds for all $\alpha\in\cI_S(x_F)$, thus $\phi-\psi \in D_F$,
and since $D_F$ is a convex cone, for any $c>0$,
$c(\phi-\psi) \in D_F$.

($\supseteq$)  Define
\[
\delta := \min\left\{ \|x_F - z\|_2 : \alpha \in \scrH_S \setminus \cI_S(x_F),
z\in\R^m, g_\alpha(z) = 0\right\}.
\]
For any fixed $\alpha$, this minimum is positive since $g_\alpha(x_F) < 0$,
while polyhedrality of $S$ grants that $\alpha$ ranges over a finite set,
together meaning $\delta >0$.
Now let any $v\in D_F$ be given, and set $\phi := x_F + \delta v / (2\|v\|_2)$.
The form of $D_F$ immediately grants $g_\alpha(\phi) \leq 0$ for $\alpha\in \cI_S(x_F)$,
but notice for $\alpha \in \scrH_S\setminus \cI_S(x_F)$, it still holds
that $g_\alpha(\phi) \leq 0$, since $g_\alpha(x_F) < 0$ 
and $\|\phi - x_F\|_2 < \delta$. 
So $v = (2\|v\|_2/\delta)(\phi - \sfP^2_{S\cap K}(\phi))$
where $\phi \in S$ and $\sfP^2_{S\cap K}(\phi) = x_F \in F$, meaning $v \in P_F$.
\end{proof}

The result now follows by considering all elements of $\cC_{S\cap \Ker(A^\top)}$.
\begin{proof}[Proof of \jmlrcref{fact:gamma2:power}]
For convenience,  set $K:=\Ker(A^\top)$.
Note that $K$ (and hence $S\cap K$) is a polyhedron; indeed, it has the form
\begin{align*}
K &= \Ker(A^\top) = \{\phi\in\R^m : A^\top \phi = \bfz_n\}
\\
&= \bigcap_{i=1}^n \left(
\{\phi\in\R^m : \bfe_i^\top A^\top \phi \leq 0\}
\cap
\{\phi\in\R^m : \bfe_i^\top A^\top \phi \geq 0\}
\right).
\end{align*}
Next, note $\cC_{S\cap K}$ has at least one
nonempty equivalence class, since $S\cap K$ is nonempty by assumption.
Rewriting \eqref{eq:gamma2:power} as in 
\eqref{eq:gamma2:precursor}, and fixing an $x_F$ within each nonempty
$F\in \cC_{S\cap K}$,
\jmlrcref{fact:gamma2:proj_onto_inside} grants
\begin{align*}
\eqref{eq:gamma2:power}
&=
\inf\left\{
\|A^\top v\|_2 : \|v\|_2 =1, \exists c >0, \exists \phi \in S\centerdot
\phi -\sfP^2_{S\cap K}(\phi) = cv
\right\}
\\
&=
\min_{\substack{F \in \cC_{S\cap K}\\F\neq \emptyset}}
\inf\left\{
\|A^\top v\|_2 : \|v\|_2 =1, \exists c >0, \exists \phi \in S\centerdot
\phi -\sfP^2_{S\cap K}(\phi) = cv, \sfP^2_{S\cap K}(\phi) \in F
\right\}
\\
&=
\min_{\substack{F \in \cC_{S\cap K}\\F\neq \emptyset}}
\inf\left\{
\|A^\top v\|_2 : \|v\|_2 =1, 
v \in N_F(x_F), \forall \alpha\in \cI_S(x_F)\centerdot g_\alpha(x_F+v) \leq 0
\right\}.
\end{align*}
Since $S\setminus \Ker(A^\top)\neq \emptyset$ and $S\cap \Ker(A^\top)$,
at least one infimum has a nonempty domain (for the others, take the convention that
their value is $+\infty$).
Each infimum with a nonempty domain in this final expression
is of a continuous function over a compact
set (in fact, a polyhedral cone intersected with the boundary of the unit $l^2$ ball),
and thus it has a minimizer $\bar v$, which corresponds to some
$c(\bar\phi - \sfP^2_{S\cap K}(\bar \phi)) \not \in \Ker(A^\top)$, where $c>0$.
It follows that
\[
A^\top \bar v = c A^\top (\bar \phi - \sfP^2_{S\cap K}(\bar\phi)) \neq 0,
\]
meaning each of these infima is positive.
But since $S$ is polyhedral, $\cC_S$ has finitely
many equivalence classes ($|\cC_S| \leq 2^{|\scrH_S|}$), meaning the outer minimum is attained and positive.
\end{proof}

Finally, as mentioned above, the desired result follows by norm equivalence.
\begin{proof}[Proof of \Cref{fact:gamma_p:sanity_check}]
For the upper bound,
note as in the proof of \jmlrcref{fact:gamma2:power} that
$S\cap\Ker(A^\top)\neq \emptyset$
and
the infimand is positive for every element of the domain, so the infimum is finite.
For the lower bound, 
by \jmlrcref{fact:gamma2:power} and norm equivalence,
\begin{align*}
\gamma(A,S) 
&=
\inf_{\phi \in S\setminus \Ker(A^\top)}
\frac {\|A^\top\phi\|_\infty}{\inf_{\psi \in S\cap\Ker(A^\top)}\|\phi-\psi\|_1}
\\
&
\geq
\left(
\frac {1}{\sqrt{mn}}
\right)
\inf_{\phi \in S\setminus \Ker(A^\top)}
\frac {\|A^\top\phi\|_2}{\inf_{\psi \in S\cap\Ker(A^\top)}\|\phi-\psi\|_2}
>0.
\tag*{\qedhere}
\end{align*}
\end{proof}

\section{Miscellaneous Technical Material}
\label{sec:deferred_proofs}

\subsection{The Logistic Loss is within $\bG$}
\begin{jremark}
\label{rem:logistic_loss_in_bG}
This remark develops bounds on the quantities $\eta,\beta$ for the logistic loss
$g = \ln(1+\exp(\cdot))$.  First note that the initial level  set
$S_0 := \{ x\in \R^m : f(x) \leq f(A\lambda_0)\}$
is contained within
a cube $(-\infty,b]^m$, where $b \leq m\ln(2)$; this follows since 
$f(A\lambda_0) = f(\bfz_m) = m\ln(2)$, whereas 
$g(m\ln(2)) = \ln(1+\exp(m\ln(2)))\geq m\ln(2)$.

For convenience, the analysis will be mainly written with respect to $b=m\ln(2)$.
Let any $x\in(-\infty,b]$ be given, and note $g' = \exp(\cdot)/(1+\exp(\cdot))$, and 
$g'' = \exp(\cdot) / (1+\exp(\cdot))^2$.

To determine $\eta$, note
 $1 \leq 1 + \exp(x) \leq 1+\exp(b)$.  
Since $\ln$ is concave, it follows for all $z\in[1,1+\exp(b)]$
that the secant
line through $(1,0)$ and $(1+\exp(b),\ln(1+\exp(b)))$ is a lower bound:
\[
\ln(z) 
\geq \left(\frac{\ln(1+\exp(b))-0}{1+\exp(b)-1}\right)z - \frac{\ln(1+\exp(b))-0}{1+\exp(b)-1}
= \ln(1+\exp(b))\exp(-b)(z-1).
\]
As such, for $x\in (-\infty,b]$,
$\ln(1+\exp(x)) \geq \exp(x) \ln(1+\exp(b)) \exp(-b)$,
so
\[
\frac{g''(x)}{g(x)} = \frac{\exp(x)}{(1+\exp(x))^2\ln(1+\exp(x))}
\leq \frac {\exp(b)}{(1+\exp(x))^2\ln(1+\exp(b))}
\leq \frac {\exp(b)}{\ln(1+\exp(b))}.
\]
Consequently, a sufficient choice is
$\eta := \exp(b) / \ln(1+\exp(b)) \leq 2^m / (m\ln(2))$.

For $g(x) \leq \beta g'(x)$, using $\ln(x) \leq x-1$,
\[
\frac {g(x)}{g'(x)} 
= \frac {\ln(1+\exp(x))}{\frac {\exp(x)}{1+\exp(x)}}
\leq \frac {\exp(x)}{\frac {\exp(x)}{1+\exp(x)}}
\leq 1 + \exp(b).
\]
That is, it suffices to set $\beta := 1+\exp(b)= 1+2^m$.
\end{jremark}

\subsection{Proof of \Cref{fact:primal_dual}}
\label{sec:proof:fact:primal_dual}
\begin{proof}[Proof of \jmlrcref{fact:primal_dual}]
Writing the objective as two Fenchel problems,
\begin{align*}
\bar f_A =& \inf_\lambda f(A\lambda) + \iota_{\R^n}(\lambda), \\
d :=& \sup_\phi -f^*(-\phi) - \iota_{\R^n}^*(A^\top\phi).
\end{align*}
Since $\cont(f) = \R^m$ (set of points where $f$ is continuous) 
and $\dom(\iota_{\R^n}) = \R^n$,
it follows that $A\dom(\iota_{\R^n}) \cap \cont(f) = \Im(A) \neq \emptyset$,
thus $d=\bar f_A$
\ifjmlr
\citep[Theorem 3.3.5]{borwein_lewis}.
\else
\citep[Theorem 3.3.5]{borwein_lewis}.
\fi
Moreover, since 
$\bar f_A \leq f(\bfz_m)$ and $d \geq -f^*(\bfz_m) = 0$,
the optimum is finite, and thus the same theorem grants that it is attainable
in the dual.

To complete the dual problem, note for any $\lambda\in \R^n$ that
\begin{align*}
\iota_{\R^n}^*(\lambda)
&= \sup_{\mu\in\R^n} \ip{\lambda}{\mu} - \iota_{\R^n}(\mu)
= \iota_{\{\bfz_n\}}(\lambda).
\end{align*}
From this, the term $-\iota_{\R^n}^*(A^\top \phi)$ allows the search in the 
dual to be restricted to $\phi \in \Ker(A^\top)$.  Next, replace $\phi \in \Ker(A^\top)$
with $-\psi \in \Ker(A^\top)$, which combined with
$\dom(f^*) \subseteq \R^m_+$ (from \jmlrcref{fact:fprop})
means it suffices to consider $\psi \in \Ker(A^\top)\cap \R^m_+ = \Phi_A$.  (Note that the
negation was simply to be able to interpret feasible dual variables as nonnegative
measures.)

Next, $f^*(\phi) = \sum_i g^*((\phi)_i)$ was proved in \jmlrcref{fact:fprop}.

Finally,
the uniqueness of $\psi_A^f$ was established by
\ifjmlr
\citet[Theorem 1]{collins_schapire_singer_adaboost_bregman},
\else
\citet[Theorem 1]{collins_schapire_singer_adaboost_bregman},
\fi
however a direct argument is as follows by the strict convexity of
$f^*$ (cf. \jmlrcref{fact:fprop}).
Specifically, if there were some other optimal $\psi'\neq \psi$, the point
$(\psi + \psi')/2$ is dual feasible and has strictly larger objective value, a 
contradiction.
\end{proof}

\subsection{Proof of \Cref{fact:strictconvex_0coercive_attainable}}
\label{sec:proof:fact:strictconvex_0coercive_attainable}
\begin{proof}[Proof of \jmlrcref{fact:strictconvex_0coercive_attainable}]
It holds in general that 0-coercivity grants attainable minima (cf.
\ifjmlr
\citet[Proposition B.3.2.4]{HULL} and \citet[Proposition 1.1.3]{borwein_lewis}).
\else
\citet[Proposition B.3.2.4]{HULL} and \citet[Proposition 1.1.3]{borwein_lewis}).
\fi
Conversely, let $\bar x$ with $h(\bar x) = \inf_x h(x)$ 
and any direction $d\in \R^m$ with $ \|d\|_2 = 1$ be given.  To demonstrate 
0-coercivity, it suffices to show
\[
\lim_{t\to\infty} \frac {h(\bar x + td) - h(\bar x)}{t} > 0
\]
\ifjmlr
\citep[Proposition B.3.2.4.iii]{HULL}.
\else
\citep[Proposition B.3.2.4.iii]{HULL}.
\fi
To this end, first note, for any $t\in \R$, that convexity grants
\[
h(\bar x + td)
\geq h(\bar x + d) + (t-1)\ip{\nabla h(\bar x+d)}{d}.
\]
By strict monotonicity of gradients 
\ifjmlr
\citep[Section B.4.1.4]{HULL}
\else
\citep[Section B.4.1.4]{HULL}
\fi
and first-order necessary conditions ($\nabla h(\bar x) = \bfz_m$),
\[
\ip{\nabla h(\bar x + d)}{d} = \ip{\nabla h(\bar x + d) - \nabla h(\bar x)}{\bar x + d -\bar x} =: c > 0,
\]
Combining these,
\[
\lim_{t\to\infty} \frac {h(\bar x + td) - h(\bar x)}{t}
\geq 
\lim_{t\to\infty} \frac {h(\bar x + d) + (t-1)c - h(\bar x)}{t}
=
c > 0.
\qedhere
\]
%
\end{proof}

\subsection{Proof of \Cref{fact:attain:hypercube}}
\label{sec:proof:fact:attain:hypercube}
\begin{proof}[Proof of \jmlrcref{fact:attain:hypercube}]
Since $d \geq \inf_\lambda f(A\lambda)$, the level set 
$S_d := \{x\in\R^m : (f+\iota_{\Im(A)})(x) \leq d\}$ is nonempty.
Since $|H(A)|= m$, \jmlrcref{fact:stiemke:derived} provides 
$f+\iota_{\Im(A)}$ is 0-coercive,
meaning $S_d$ is compact.

Now consider the rectangle $\cC$ defined as a product of intervals
$\cC= \otimes_{i=1}^m [a_i,b_i]$, where
\[
a_i := \inf\{ x_i : x\in S_d\},
\quad\quad
b_i := \sup\{ x_i : x\in S_d\}.
\]
By construction, $\cC\supseteq S_d$, and furthermore any smaller axis-aligned rectangle
must violate some infimum or supremum above, and so 
must fail to include a piece of $S_d$.  In particular, the tightest rectangle
exists, and it is $\cC$.


Next, note that $\nf(x) = (g'(x_1), g'(x_2), \ldots, g'(x_m))$,
thus
$D = \otimes_{i=1}^m g'([a_i,b_i])$, an axis-aligned rectangle in the dual.
Since $g$ is strictly convex and $\dom(g) = \R$, both $g'(a_i)$ and
$g'(b_i)$ are within $\interior(\dom(g^*))$ (for all $i$), and so
$\nf(\cC) \subset \interior(\dom(f^*))$.

Finally, \jmlrcref{fact:strictconvex_0coercive_attainable} grants that 
$f+\iota_{\Im(A)}$ has a minimizer; thus choose any $\bar\lambda\in\R^n$ so that
$f(A\bar \lambda) = \inf_\lambda f(A\lambda)$.  By optimality conditions of
Fenchel problems, $\psi_A^f = \nabla f(A\bar\lambda)$ (cf. the optimality conditions
in
\citet[Exercise 3.3.9.f]{borwein_lewis},
and the proof of 
\jmlrcref{fact:primal_dual}, where a negation was inserted into the dual to allow
dual points to be interpreted as nonnegative measures).
But the dual optimum is dual feasible, and $A\bar\lambda \in S_d$, so 
\[
\nf(\cC) \cap \Phi_A
\supseteq \{\nf(A\bar\lambda)\} \cap \Phi_A
= \{\psi_A^f\} \cap \Phi_A
\neq \emptyset.
\qedhere
\]
\end{proof}

\subsection{Splitting Distances along $A_0,A_+$}
\label{sec:proof:fact:interp:norm_split}
\begin{jlemma}
\label{fact:interp:norm_split}
Let $A=\left[\begin{smallmatrix}A_0 \\ A_+\end{smallmatrix}\right]$ be given as
in \jmlrcref{fact:rate:interp}, and let a set $S = S_0\times S_+$ be given with
$S_0\subseteq \R^{m_0}$ and $S_+\subseteq \R^{m_+}$
and $S\cap\Phi_A\neq\emptyset$.  Then, for any
$\phi = \left[\begin{smallmatrix}\phi_0 \\ \phi_+\end{smallmatrix}\right]$ with
$\phi_0 \in \R^{m_0}$ and $\phi_+ \in \R^{m_+}$,
\[
\dist^1_{S\cap\Phi_A}(\phi) = 
\dist^1_{S_0\cap \Phi_{A_0}}(\phi_0)
+\dist^1_{S_+\cap \Phi_{A_+}}(\phi_+).
\]
\end{jlemma}
\begin{proof}
Recall from \jmlrcref{fact:megagordan} that $\Phi_A = \Phi_{A_0}\times \Phi_{A_+}$,
thus
\[
S\cap\Phi_A = (S_0\cap \Phi_{A_0}) \times (S_+\cap \Phi_{A_+}),
\]
and $S\cap\Phi_A\neq \emptyset$ grants that $S_0\cap\Phi_{A_0}\neq \emptyset$
and $S_+\cap\Phi_{A_+}\neq \emptyset$.  Define now the notation 
$[\cdot]_0 :\R^m \to \R^{m_0}$ and $[\cdot]_+ :\R^m\to\R^{m_+}$, which respectively
select the coordinates corresponding to the rows of $A_0$, and the rows of $A_+$.

Let
$\phi = \left[\begin{smallmatrix}\phi_0 \\ \phi_+\end{smallmatrix}\right]\in\R^m$ be
given; in the above notation, $\phi_0 = [\phi]_0$ and $\phi_+ = [\phi]_+$.
By the above Cartesian product and intersection properties,
\[
\begin{bmatrix}
\sfP^1_{S_0\cap\Phi_{A_0}}(\phi_0) \\
\sfP^1_{S_+\cap\Phi_{A_+}}(\phi_+)
\end{bmatrix}
\in S\cap \Phi_A,
\]
and so
\begin{align*}
\dist^1_{S\cap\Phi_A}(\phi) 
&\leq
\left\|
\left[\begin{smallmatrix}\phi_0 \\ \phi_+\end{smallmatrix}\right]
-
\left[\begin{smallmatrix}
\sfP^1_{S_0\cap\Phi_{A_0}}(\phi_0) \\
\sfP^1_{S_+\cap\Phi_{A_+}}(\phi_+)
\end{smallmatrix}\right]
\right\|_1
=
\dist^1_{S_0\cap \Phi_{A_0}}(\phi_0)
+\dist^1_{S_+\cap \Phi_{A_+}}(\phi_+).
\end{align*}
On the other hand, since
$\sfP^1_{S\cap\Phi_A}(\phi) \in (S_0\cap \phi_{A_0}) \times (S_+\cap \phi_{A_+})$,
\[
\dist^1_{S_0\cap \Phi_{A_0}}(\phi_0)
+\dist^1_{S_+\cap \Phi_{A_+}}(\phi_+)
\leq
\left\|\phi_0 - [\sfP^1_{S\cap\Phi_A}(\phi)]_0\right\|_1
+\left\|\phi_+ - [\sfP^1_{S\cap\Phi_A}(\phi)]_+\right\|_1
=
\dist^1_{S\cap \Phi_{A}}(\phi).
\qedhere
\]
\end{proof}

\subsection{Proof of \Cref{fact:rate:lb}}
\label{sec:proof:fact:rate:lb}
\begin{proof}[Proof of \jmlrcref{fact:rate:lb}]
This proof proceeds in two stages: first the gap between any solution with $l^1$ norm $B$
is shown to be large, and then it is shown that the $l^1$ norm of the \textsc{Boost}
solution (under logistic loss) grows slowly.

To start, $\Ker(S^\top) = \{z(1,1,0) : z\in \R\}$, and $-g^*$ is maximized
at $g'(0)$  with value $-g(0)$ (cf. \jmlrcref{fact:gconj_prop}). 
Thus $\psi_S^f = (g'(0),g'(0),0)$, and 
$\bar f_S = -f^*(\psi_S^f) = 2g(0) = 2\ln(2)$.

Next, by calculus, given any $B$,
\begin{align*}
\inf_{\|\lambda\|_1\leq B} f(S\lambda) - \bar f_S
&= f \left(
S \left[
\begin{smallmatrix}
B/2 \\ B/2
\end{smallmatrix}
\right]
\right)
- 2\ln(2)
\\
&= (2\ln(2) + \ln(1+\exp(-B))) - 2\ln(2)
\\
&= \ln(1+\exp(-B)).
\end{align*}

Now to bound the $l^1$ norm of the iterates.  By the nature of exact line search,
the coordinates of $\lambda$ are updated in alternation (with arbitrary initial
choice); thus let $u_t$ denote the value of the coordinate updated in iteration $t$,
and $v_t$
be the one which is held fixed.  (In particular, $v_t = u_{t-1}$.)

The objective function, written in terms of $(u_t,v_t)$, is
\begin{align*}
&\ln\big(1 + \exp(v_t - u_t)\big) + \ln\big(1+ \exp(u_t - v_t)\big) + \ln\big(1 + \exp(-u_t - v_t)\big)
\\
=\quad
& \ln\big(2 + \exp(v_t-u_t) + \exp(u_t-v_t) + 2\exp(-u_t-v_t) + \exp(-2u_t) + \exp(-2v_t)\big).
\end{align*}
Due to the use of exact line search, and the fact that $u_t$ is the new value
of the updated variable, the derivative with respect to $u_t$ of the above expression
must equal zero.  In particular, producing this equality and multiplying both sides
by the (nonzero) denominator  yields
\[
-\exp(v_t-u_t) + \exp(u_t-v_t) -2\exp(-u_t-v_t) -2 \exp(-2u_t) = 0.
\]
Multiplying by $\exp(u_t + v_t)$ and rearranging, it follows that, after line search,
$u_t$ and $v_t$ must satisfy
\begin{equation}
\exp(2u_t) = \exp(2v_t) + 2\exp(v_t - u_t) + 2.
\label{eq:logloss:lb:helper_dog}
\end{equation}

First it will be shown for $t\geq 1$,
by induction, that $u_t \geq v_t$.
The base case follows by inspection (since $u_0 = v_0 = 0$ and so $u_1 = \ln(2)$).  
Now the inductive hypothesis grants $u_t \geq v_t$; the case $u_t=v_t$ can be directly
handled by \jmlreqref{eq:logloss:lb:helper_dog}, thus suppose $u_t > v_t$.
But previously, it was shown that the optimal $l^1$ bounded choice has both coordinates
equal; as such, the current iterate, with coordinates $(u_t,v_t)$, is worse than the
iterate $(u_t,u_t)$, and thus the line search will move in a positive direction,
giving $u_{t+1} \geq v_{t+1}$.

It will now be shown  by induction that, for $t\geq 1$,
$u_t \leq \frac 1 2 \ln(4t)$.  The base case follows by the direct inspection above.
Applying the inductive hypothesis to the update rule above, and recalling
$v_{t+1} = u_{t}$ and that the weights increase (i.e., $u_{t+1} \geq v_{t+1} = u_t$),
\[
\exp(2u_{t+1}) = \exp(2u_t) + 2\exp(u_t-u_{t+1}) + 2
\leq \exp(2u_t) + 2\exp(u_t - u_t) + 2
\leq 4t + 4
\leq 4(t+1).
\]

To finish, recall by Taylor expansion that $\ln(1+q) \geq q - \frac {q^2} {2}$;
consequently for $t\geq 1$
\[
f(S\lambda_t) - \bar f_S
\geq \inf_{\|\lambda\|_1 \leq \ln(4t)} f(S\lambda)- \bar f_S
\geq \ln\left(1+ \frac 1 {4t}\right)
\geq \frac 1 {4t} - \frac 1 2 \left(\frac 1 {4t}\right)^2
\geq \frac 1 {8t}.
\qedhere
\]
\end{proof}

\ifjmlr
\addcontentsline{toc}{section}{References}
\bibliography{ab}
\fi

\end{document}